\documentclass[acmsmall]{acmart}

\usepackage{hyperref}
\usepackage{color}
\usepackage{makecell}
\usepackage{url}
\usepackage{balance}
\usepackage{tabularx}

\usepackage{bm}
\usepackage{pifont}
\newcommand{\xmark}{\ding{55}}

\usepackage{algorithmic}
\usepackage{graphicx}
\usepackage{subcaption}
\usepackage[boxed,ruled,lined]{algorithm2e}
\usepackage{amsthm}
\usepackage[american]{babel}
\usepackage{setspace}
\usepackage{enumerate}
\usepackage{multirow}
\usepackage{url}

\usepackage{graphicx}
\usepackage{subcaption}

\newcommand{\argmax}{\operatornamewithlimits{arg\,max}}
\newcommand{\argmin}{\operatornamewithlimits{arg\,min}}

\AtBeginDocument{%
  \providecommand\BibTeX{{%
    \normalfont B\kern-0.5em{\scshape i\kern-0.25em b}\kern-0.8em\TeX}}}

\setcopyright{acmlicensed}
\copyrightyear{2018}
\acmYear{2018}
\acmDOI{XXXXXXX.XXXXXXX}

\acmJournal{JACM}
\acmVolume{37}
\acmNumber{4}
\acmArticle{111}
\acmMonth{8}

\acmConference[Conference acronym 'XX]{Make sure to enter the correct
  conference title from your rights confirmation emai}{June 03--05,
  2018}{Woodstock, NY}
\acmISBN{978-1-4503-XXXX-X/18/06}

\begin{document}

\title{Enhancing Bandit Algorithms with LLMs for Time-varying User Preferences in Streaming Recommendations}
\def\shorttitle{LLM4Bandit}


\author{Chenglei Shen}
\email{chengleishen9@ruc.edu.cn}
\affiliation{%
  \institution{Gaoling School of Artificial Intelligence, Renmin University of China}
  \city{Beijing}
  \country{China}
}

\author{Yi Zhan}
\email{zistrawberry1223@gmail.com}
\affiliation{%
  \institution{Gaoling School of Artificial Intelligence, Renmin University of China}
  \city{Beijing}
  \country{China}
}

\author{Weijie Yu}
\email{yu@uibe.edu.cn}
\affiliation{%
  \institution{School of Artificial Intelligence and Data Science, University of International Business and Economics}
  \city{Beijing}
  \country{China}
}
\author{Xiao Zhang}
\authornote{Corresponding author (e-mail: zhangx89@ruc.edu.cn)}
\email{zhangx89@ruc.edu.cn}
\affiliation{%
  \institution{Gaoling School of Artificial Intelligence, Renmin University of China}
  \city{Beijing}
  \country{China}
}
\author{Jun Xu}
\email{junxu@ruc.edu.cn}
\affiliation{%
  \institution{
Gaoling School of Artificial Intelligence, Renmin University of China}
  \city{Beijing}
  \country{China}
}




\renewcommand{\shortauthors}{Trovato and Tobin, et al.}

\begin{abstract}
In real-world streaming recommender systems, user preferences evolve dynamically over time. Existing bandit-based methods treat time merely as a timestamp, neglecting its explicit relationship with user preferences and leading to suboptimal performance. Moreover, the online learning methods often suffer from inefficient exploration–exploitation during the early online phase. To address these issues, we propose HyperBandit+, a novel contextual bandit policy which integrates a time-aware hypernetwork to adapt to time-varying user preferences and employs a large language model-assisted warm-start mechanism (LLM Start) to enhance exploration–exploitation efficiency at the early online phase. Specifically, HyperBandit+ leverages a neural network that takes time features as input and generates parameters for estimating time-varying rewards by capturing the correlation between time and user preferences. Additionally, the LLM Start mechanism employs multi-step data augmentation to simulate realistic interaction data for effective offline learning, providing warm-start parameters for the bandit policy at the early online phase. To meet real-time streaming recommendation demands, we adopt low-rank factorization to reduce hypernetwork training complexity. Theoretically, we rigorously establish a sublinear regret upper bound that accounts for both the hypernetwork and the LLM warm-start mechanism. Extensive experiments on real-world datasets demonstrate that HyperBandit+ consistently outperforms state-of-the-art baselines in terms of accumulated rewards. The code is available at \url{https://github.com/Starrylay/HyperBandit.git}.

\end{abstract}

\ccsdesc[300]{Information systems~Recommender systems}

\keywords{Contextual bandit, Large language models, Hypernetwork, Streaming recommendation, Time-varying user preference}


\maketitle

\section{Introduction}
\label{sec:intro}


The exponential growth of the internet industry has precipitated an unprecedented surge in online platforms and applications, intensifying challenges associated with information overload. In response, recommender systems have evolved into indispensable solutions for mitigating this issue by curating contextually relevant content for users. Leveraging sophisticated architectures, these systems generate personalized recommendations that optimize user engagement in digital ecosystems \cite{chen2022graph, tian2022learning, wei2022contrastive}. A robust recommender framework necessitates meticulous analysis of historical user-item interactions to uncover latent behavioral patterns. Concurrently, the dynamic nature of online environments demands rapid adaptation to shifting user preferences \cite{zou2022multi}, requiring iterative model updates to sustain real-time responsiveness. This operational paradigm has propelled streaming recommendation into a prominent research domain, emphasizing incremental model optimization through immediate integration of users’ latest platform interactions. The ultimate objective is to deliver temporally aligned recommendations that mirror evolving user preferences and behavioral trajectories \cite{Chandramouli2011StreamRec, Chang2017Streaming, Jakomin2020Simultaneous, Wang2018Neural, Zhang2022Counteracting, Wang2018Streaming}.


\begin{figure}[!t]
\centering
\begin{subfigure}[b]{0.24\textwidth}
\includegraphics[width=\textwidth]{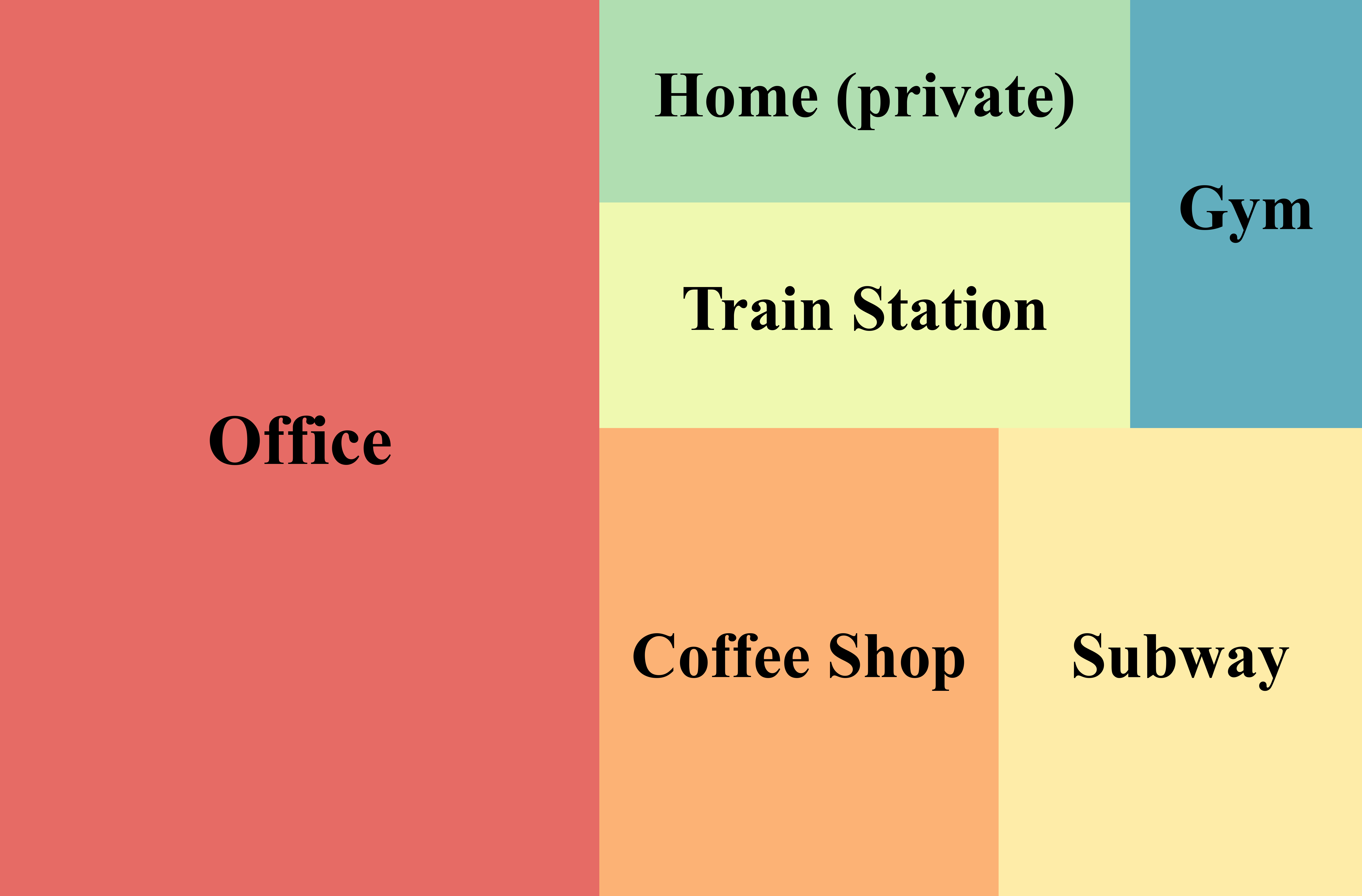}
\caption{\texttt{weekday morning}}
\end{subfigure}
\hfill
\begin{subfigure}[b]{0.24\textwidth}
\includegraphics[width=\textwidth]{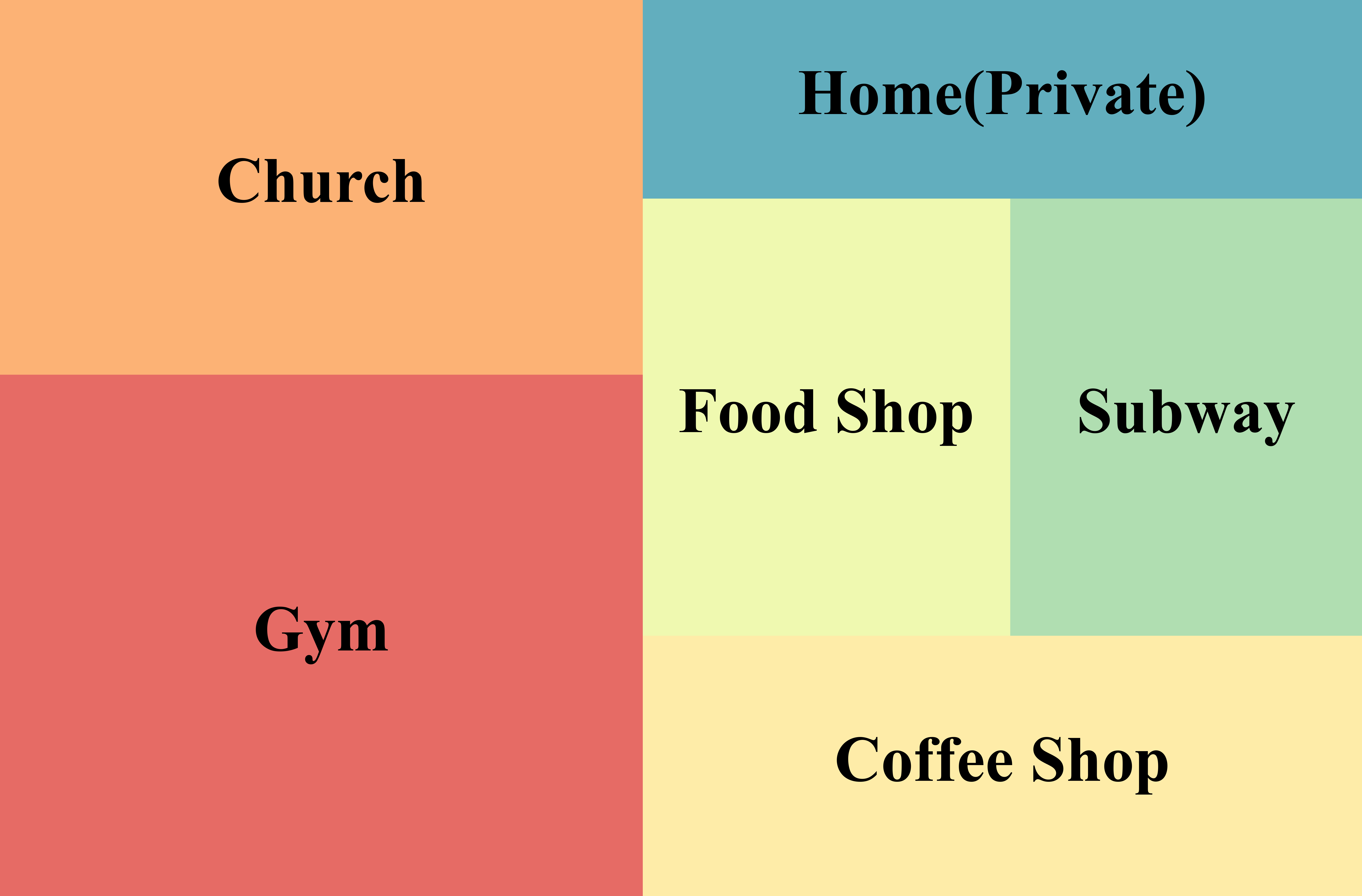}
\caption{\texttt{weekend morning}}
\end{subfigure}
\hfill
\begin{subfigure}[b]{0.24\textwidth}
\includegraphics[width=\textwidth]{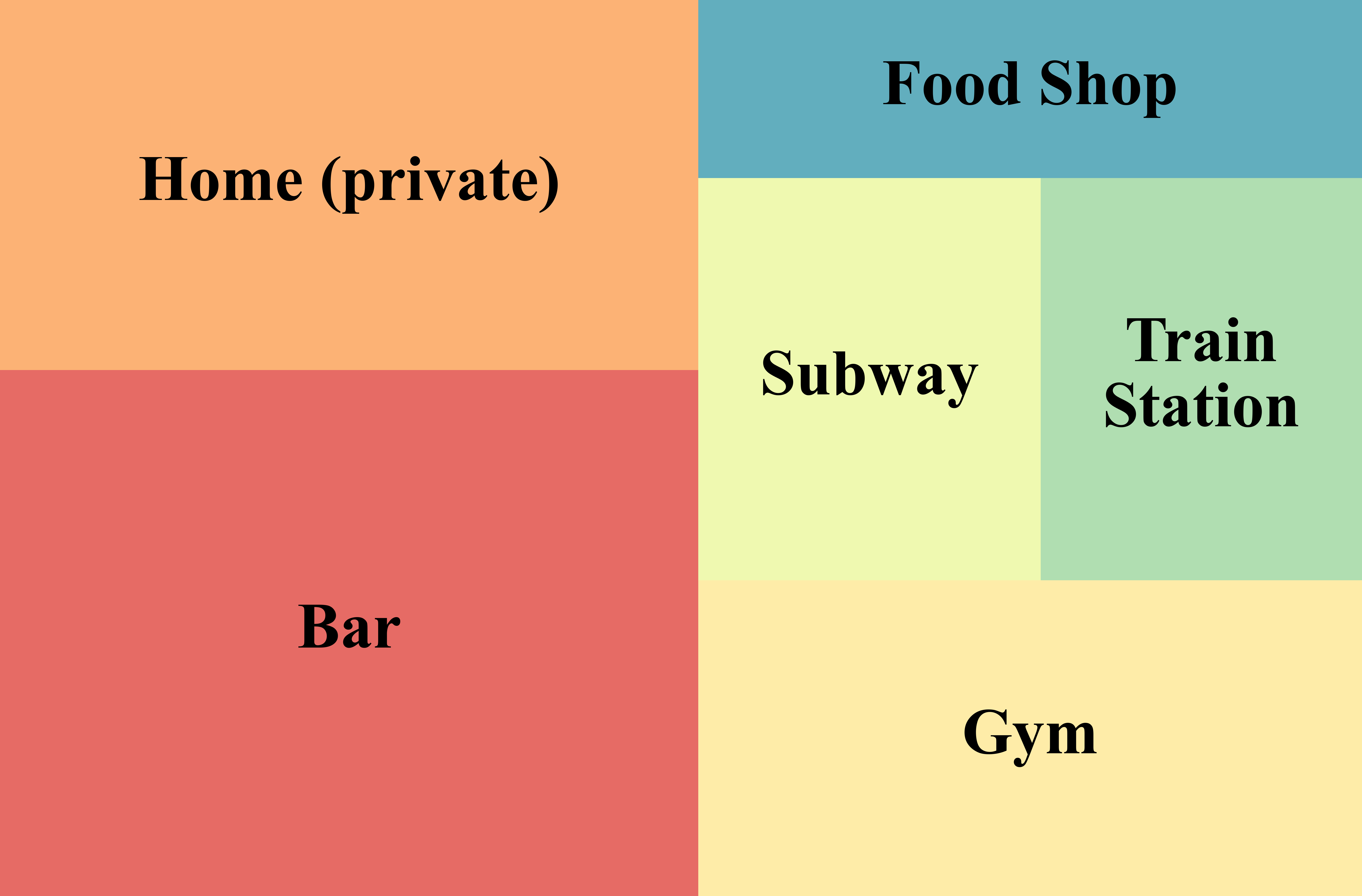}
\caption{\texttt{weekday night}}

\end{subfigure}
\hfill
\begin{subfigure}[b]{0.24\textwidth}
\includegraphics[width=\textwidth]{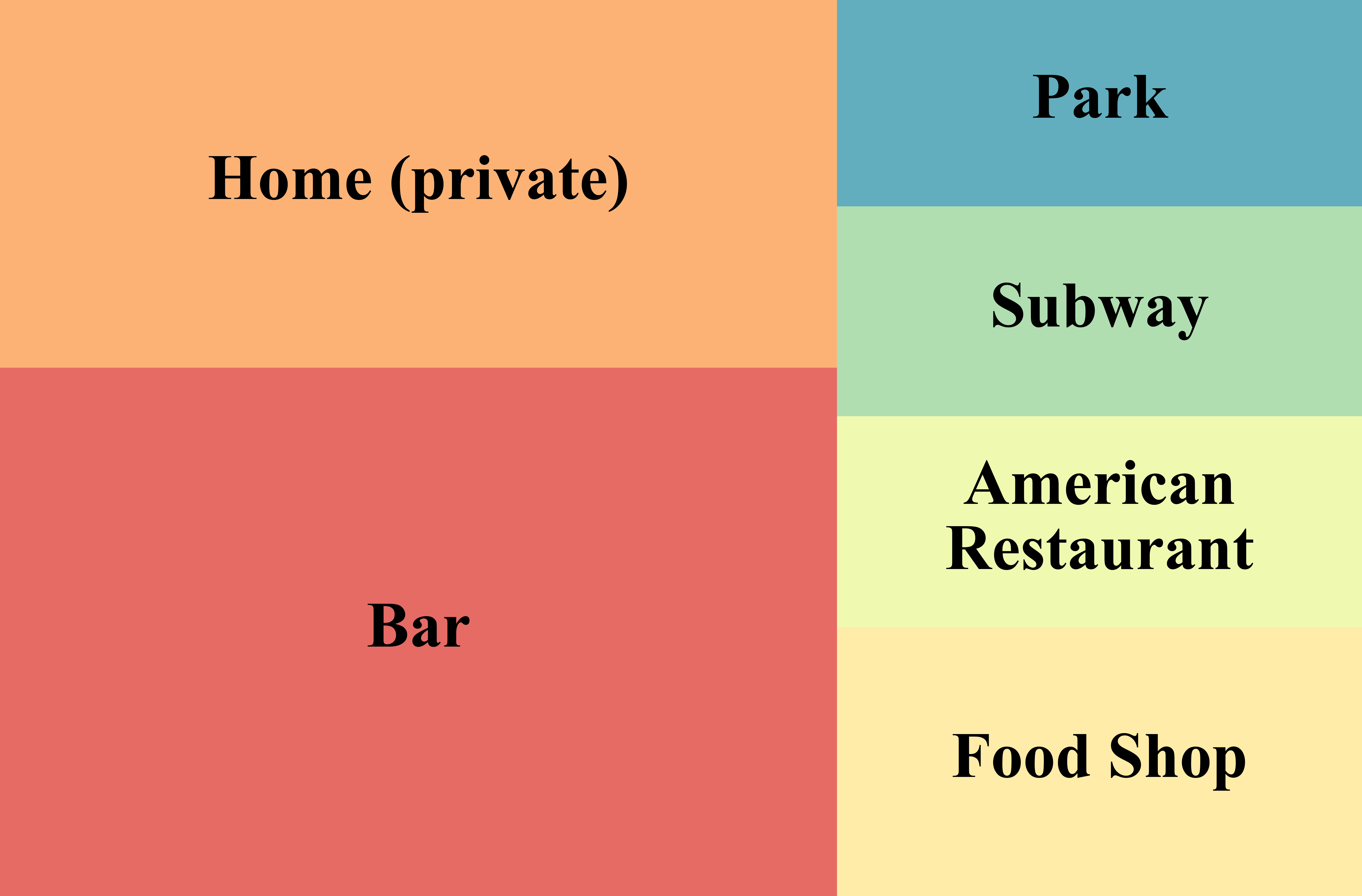}
\caption{\texttt{weekend night}}
\end{subfigure}
\caption{Treemap visualizations illustrating periodic shifts in the Foursquare-NYC points-of-interest (POI) dataset for morning and night periods on weekdays and weekends. Block sizes correspond to the frequency of visits, with larger blocks indicating higher visit counts.}
\label{fig:NYCtimevarying}
\end{figure}

However, streaming recommendation faces a significant challenge due to time-varying user preferences \cite{ditzler2015learning}, which limits the model's controllability at test time \cite{shen2024survey}. User preferences undergo dynamic changes influenced by various factors including seasonality, holidays, and circadian rhythm. This temporal variability is visually depicted in Fig.\ref{fig:NYCtimevarying}, where users exhibit a pattern of checking in at places like ``Office'' and ``Coffee Shop'' during weekday mornings, contrasting with preferences for places such as ``Gym'' and ``Church'' during weekend mornings, revealing a discernible weekly periodicity.  In contrast to their morning preferences, users consistently exhibit a tendency to visit places like ``Bars'' or spend time at ``Home'' during evening hours, regardless of whether it is a weekday or weekend, highlighting a distinct daily periodicity. Similarly, an interesting illustration within the realm of short video recommendation (e.g., Kuai application) is the observable trend where users demonstrate a preference for watching cartoons predominantly during weekends. Conversely, on weekdays, their content preferences lean towards other genres. These recurrent patterns significantly impair the performance of online learning algorithms, leading to suboptimal results. For example, as a classic framework for online learning, the multi-armed bandit (MAB) algorithms is widely used in online streaming recommendation. However, most existing contextual bandit algorithms assume a stationary environment~\cite{li2010contextual,Balseiro2019Contextual,han2020sequential}, overlooking the reality of non-stationary, time-varying user preferences. While some studies relax this assumption to a piecewise stationary environment~\cite{wu2018learning,xu2020contextual,besbes2019optimal}, enabling extra detection of preference changes, they may exhibit performance fluctuations when handling periodic changes. The gap in research lies in addressing streaming recommendations in periodic environments, where current approaches fail to recognize the periodicity of user preferences in an online manner.


Furthermore, online learning algorithms simultaneously serve users while learning from their interactions, a process that typically necessitates a substantial amount of interaction data to attain an optimal state. This often results in a suboptimal user experience during the initial phase of the algorithm, primarily due to sparse rewards stemming from the expansive action space (e.g., the candidate item space) and the relatively limited number of user interactions. Fig.\ref{fig:datasparsity} illustrates the distribution of user-item interactions in the initial week. It reveals that half of the items were visited fewer than 12 times while half of the users give a visit fewer than 9 times, highlighting a pronounced issue of data sparsity. In other words, the algorithm struggles to rapidly achieve convergence, highlighting the issue of inefficient exploration and exploitation in the initial stage. 
 
Recent methodologies focus on integrating side information as auxiliary data to mitigate data sparsity, including demographic data~\cite{barragans2010hybrid}, social relationships~\cite{herce2020new, hong2018multi, vatani2023personality}, item relationships~\cite{zhang2021mining}, and other sources. Nevertheless, these methods
pose challenges when addressing data sparsity in recommender systems.  Inappropriate use of side information may introduce data noise, where attributes or features may lack direct relevance to user preferences. Furthermore, data incompleteness~\cite{ko2022survey} arises when side information lacks certain attributes or features, such as privacy concerns limiting the collection of sufficient user profiles to understand their interests. This incompleteness undermines the model's ability to fully capture the unique characteristics of users and items, thereby affecting recommendation accuracy. Hence, the urgent challenge lies in generating accurate side information and seamlessly integrating it into downstream recommendation models.


\begin{figure}[!t]
\centering
\begin{subfigure}[b]{0.488\textwidth}
\includegraphics[width=\textwidth]{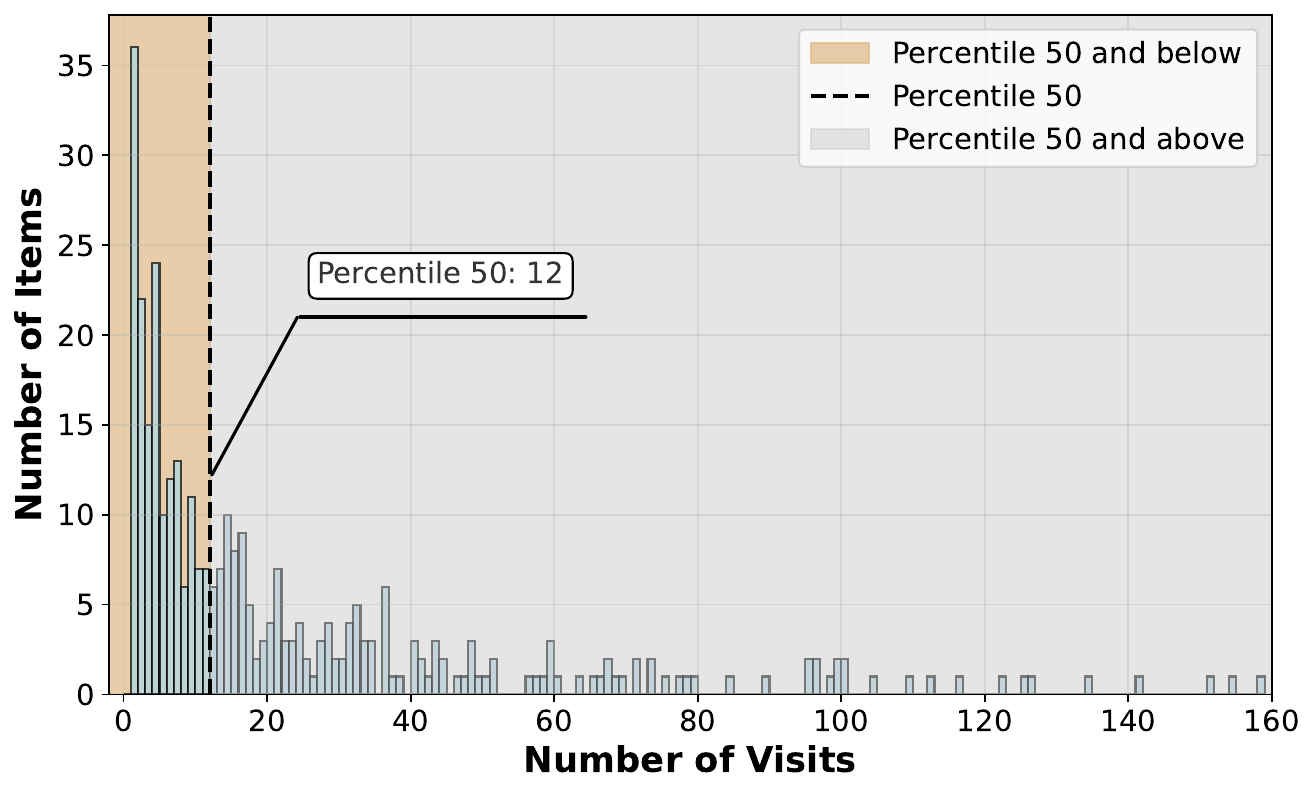}
\caption{\texttt{Distribution of Visit Number on Item.}}
\end{subfigure}
\hfill
\begin{subfigure}[b]{0.48\textwidth}
\includegraphics[width=\textwidth]{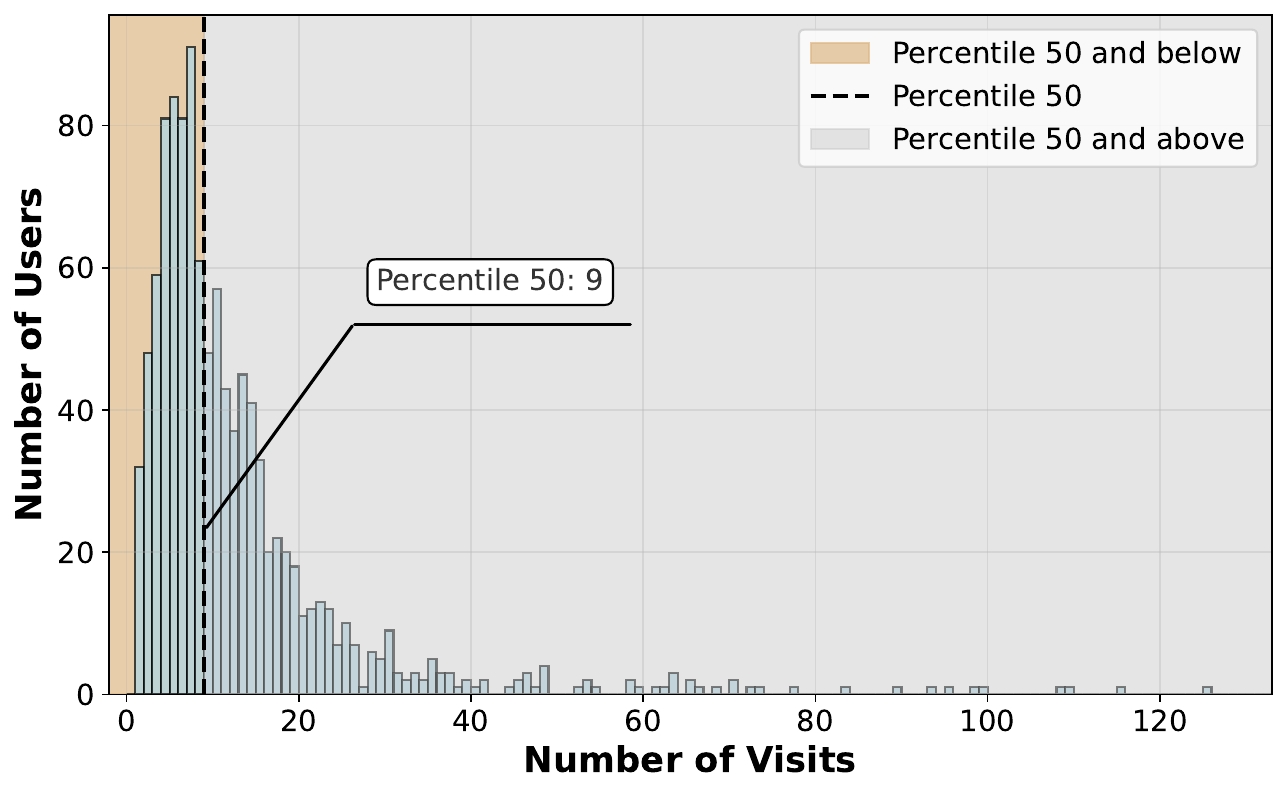}
\caption{\texttt{Distribution of Visit Number on User.}}
\end{subfigure}
\caption{The illustrations of the data sparsity.}
\label{fig:datasparsity}
\end{figure}


In this paper, we propose effective methods aiming at solving these two challenges: \textbf{(1) The challenge of periodic user preferences.} A large time period can be divided into multiple smaller periods in a periodic manner (e.g., based on the specific day of the week and different time slots within a day), and the reward function demonstrates a similar distribution whenever the same time period is encountered. Moreover, these time periods can be observed by the model and utilized for periodicity modeling and online adjustment of its user preference module in various streaming recommendation scenarios. As a targeted solution to the aforementioned process, we propose leveraging a hypernetwork that takes time period information as input to generate the parameters of the user preference matrix, thereby supporting the online learning method (i.e., the bandit policy) in adjusting its strategy. Consequently, the hypernetwork effectively captures the periodicity of user preferences over time. (2) \textbf{The challenge of inefficient exploration and exploitation during the initial stage.} Large Language Models (LLMs), trained on extensive real-world knowledge, possess the capability to comprehend user preferences and provide valuable completion information, even when dealing with privacy-constrained user profiles. Therefore, we suggest employing LLMs for an offline warm start before
online learning. Specifically, we employ LLMs for step-by-step data augmentation. Initially, we enhance user and item side information (e.g., profiles, attributes), thereby enriching the side information and alleviating the user privacy issues. Subsequently, utilizing the enriched user and item data, we obtain a substantial amount of simulated interaction data that more effectively aligns with user preferences, which allow algorithms have more trials to explore$\&$ exploit~(E$\&$E) in the offline phrase thus improving the E$\&$E efficiency in the initial stage of online phrase. 

The contributions are summarized as follows:
\begin{enumerate}[(1)]
    \item We propose a novel contextual bandit algorithm called HyperBandit+, along with an efficient online training method utilizing low-rank factorization. HyperBandit+ explicitly models the periodic variations in user preferences and dynamically adjusts the recommendation policy based on time features.
    \item We propose a novel data augmentation method called LLM Start, which effectively solve the issues of data sparsity and low-quality by utilizing Large Language Models. LLM Start utilizes world knowledge to step-by-step construct simulated interaction data, enabling downstream algorithms to perform well through a warm start.
    \item We rigorously establish a sublinear regret guarantee for HyperBandit+, thereby ensuring the convergence of the online learning process.

\end{enumerate}


\section{Related Work}
\subsection{Hypernetworks (HNs)}  

Hypernetworks (HNs), first introduced by Ha et al. \cite{ha2016hypernetworks}, are inspired by the genotype-phenotype relationship in cellular biology, utilizing one network (the hypernetwork) to generate weights for another network (the target network). Recent research demonstrates the applicability of HNs across diverse domains, including computer vision \cite{klocek2019hypernetwork}, language modeling \cite{suarez2017language}, sequence decoding \cite{nachmani2019hyper}, continual learning \cite{von2019continual}, federated learning \cite{shamsian2021personalized}, multi-objective optimization \cite{Navon2021Learning,Chen2023Controllable}, and hyperparameter optimization \cite{mackay2019self}. Navon et al. \cite{Navon2021Learning} proposed a unified hypernetwork-based model capable of learning the Pareto front, enabling inference-time adaptability to specific objective preferences. Furthermore, von Oswald et al. \cite{von2019continual} introduced a task-aware hypernetwork approach for continual model-based reinforcement learning, facilitating network adaptability across multiple tasks while preserving previous task performance. Although extensively studied in offline scenarios, research on leveraging hypernetworks to enhance model controllability within online learning and streaming contexts remains limited.

\subsection{Bandits in Non-stationary Environment}
 Non-stationary environments in online learning have recently garnered significant theoretical and practical interest. A prevalent scenario is the abruptly changing or piecewise-stationary environment, characterized by sudden changes occurring at unknown times, with stationary periods between these points. This setting has been extensively studied within the classical context-free framework \cite{hartland2006multi,garivier2008upper,yu2009piecewise,slivkins2008adapting}. For instance, Yu et al. \cite{yu2009piecewise} introduced a windowed mean-shift detection algorithm, providing an upper regret bound of $O(\Gamma_T \log (T))$, where $\Gamma_T$ denotes the total number of actual changes up to time $T$. In contrast, the contextual bandit setting has received comparatively limited attention regarding non-stationary environments \cite{hariri2015adapting,wu2018learning,xu2020contextual}. Wu et al. \cite{wu2018learning} proposed a hierarchical bandit approach employing multiple contextual bandits to detect and respond to environmental shifts. More recently, Xu et al. \cite{xu2020contextual} addressed time-varying user preferences through a change-detection mechanism targeting preference vectors. Shen et al. \cite{shen2023hyperbandit} proposed a hypernetwork-based strategy to address periodic interests. However, the algorithm requires substantial data to converge, negatively impacting user experience during the initial stages. Additionally, some work~\cite{lu2021low, kang2022efficient} exploit the low-rank structure of the preference matrix, treating the user--item (or two-entity) preference/reward matrix as low-rank, However, their goal is to leverage low rankness in static bilinear bandits to derive regret bounds explicitly dependent on the rank, emphasizing the resulting statistical efficiency.  In contrast, we employ low rankness in a time-varying contextual bandit as a means of hypernetwork parameterization and compression, reducing the output dimension from $d_a \times d_u$ to $\tau (d_a + d_u)$, which significantly cuts training and inference time in online settings.


\subsection{LLM-enhanced Recommendation}

Traditional algorithms in recommender systems are often trained to fit over the user's logged data  \cite{kang2018self, rendle2012bpr}. The quality and variety of the input data directly  influence the performance and versatility of the models. With strong capabilities in domain generalization and language generation,  LLMs show promise in addressing these data quality issues in recommendation models \cite{fan2023recommender}. Wang et al. \cite{wang2023recagent} introduce RecAgent, a recommender system simulation paradigm utilizing LLMs. It comprises a user module for social media interaction and a recommender module offering search or recommendation lists. LLM-Rec \cite{lyu2023llm} integrates four prompting strategies to enhance personalized content recommendations, improving performance through diverse prompts and input augmentation techniques. Some studies \cite{xi2023towards,liu2023first} employs LLMs to infer users' potential intention from their historical interaction data thus improving the retrieval of relevant items. Some studies \cite{baheri2023llms, li2023exploring, wei2023llmrec} use LLMs to encode side information, creating more informative user and item representations to enhance downstream recommender models, such as MAB.  Rather than solely deploying LLMs as recommender systems, utilizing them for data augmentation to bolster recommendations emerges as a  promising strategy in the future.

\section{Problem Formulation}

\subsection{Bandit-based Streaming Recommendation}
\label{3.1}

Streaming recommendation is a sequential decision-making problem in which an online platform recommends the most relevant item \( a \in \mathcal{A} \) (e.g., videos, music, or POIs) to a user \( u \in \mathcal{U} \) in real-time. Contextual bandit algorithms are particularly effective for this setting. Here, the candidate item set \( \mathcal{A} \) serves as the action space, while the context space \( \mathcal{S} \) encapsulates feature representations of users and items. Each item \( a \) and user \( u \) is characterized by context feature vectors \( \bm{c}_a \) and \( \bm{c}_u \), respectively. At each time step \( t \), given a subset \( \mathcal{A}_t \subseteq \mathcal{A} \) and a user context, the recommendation policy $\pi$ selects an item for recommendation. The user then provides feedback by either clicking (positive) or skipping (negative) the item. This feedback determines the true reward, which is subsequently used to update the recommendation policy for future decisions.
The above process can be formalized as a contextual bandit problem for streaming recommendation, and represented using a 4-tuple $\left\langle\mathcal{A}, \mathcal{S}, \pi, r \right\rangle$:

\textbf{Action space \(\mathcal{A}\)} represents a predefined set of candidate actions, where each action (or arm) corresponds to a specific item. At each time step \( t \), a dynamic subset \( \mathcal{A}_t \subseteq \mathcal{A} \) is selected as the candidate item set for recommendation. Selecting an action \( a_{I_{t}} \) from \( \mathcal{A}_{t} \) implies recommending the corresponding item to the user, where \( I_{t} \in |\mathcal{A}_t| \) denotes the index of the recommended item at time \( t \).



\textbf{Context space \(\mathcal{S}\)} encapsulates the feature information of users and items, represented by \( \bm c_u \in \mathbb{R}^{d_u} \) and \( \bm c_a \in \mathbb{R}^{d_a} \), respectively. In this work, we specifically decompose the item context \( \bm c_a \) into observed features \( \bm s_a \in \mathbb{R}^{o_a} \) and latent features \( \bm x_a \in \mathbb{R}^{l_a} \), which require learning. This decomposition is expressed as \( \bm c_a = [\bm s_a^\intercal, \bm x_a(t)^\intercal]^\intercal \), where the total dimension of \( \bm c_a \) is given by \( d_a = o_a + l_a \).

\textbf{Policy $\pi : \mathcal{S} \rightarrow \mathcal{A}$} describes the decision-making rule of an agent (i.e., the recommendation model), which selects an action for execution according to the relevance score of each action at time $t$, given a candidate item set $\mathcal{A}_{t}$ and user $u \in \mathcal{U}$, a \emph{relevance score function} $f_t$ treats context features of user and item  in context space (i.e., $\bm c_{u}$ and $\bm c_a$ ) as inputs and determines which action to take: $a_{I_t}:=\arg \max _{a \in \mathcal{A}_{t}} f_t\left(\bm c_{u} , \bm c_a \right)$.


\textbf{Reward $r$} is determined based on user feedback. Specifically, at time $t$, after recommending item $a_{I_t} \in \mathcal{A}_t$ to user $u$, a reward $r(u, a_{I_t}) \in \{0,1\}$ is observed, implicitly indicating positive or negative user feedback toward item $a_{I_t}$. Notably, feedback from the same user regarding the same item may vary across different time steps, reflecting time-varying user preferences.

Table~\ref{tab:notations} summarizes the notations used throughout the paper.

\begin{table}[t]
    \centering
    \caption{A summary of notations.}
    \label{tab:notations}
    \begin{tabularx}{0.75\textwidth}{c|X} 
        \toprule 
        {\bf Symbol} & {\bf Explanation} \\
        \midrule 
        $[n]$ & $[n]: = [1,2, \ldots, n]$ \\
        $t$ & Time step $t \in [T]$ \\
        $\mathcal{A}$ & Action space, i.e., the candidate item set \\
        $|\mathcal{A}|$ & The cardinality of set $\mathcal{A}$ \\
        $\bm c_{u} \in \mathbb{R}^{d_u}$ & Context feature vector of a user $u$ \\
        $\bm c_{a} \in \mathbb{R}^{d_a}$ & Context feature vector of a candidate item $a$ \\
        $\bm s_{a} \in \mathbb{R}^{o_a}$ & Observed features of a candidate item $a$ \\
        $\bm x_{a} \in \mathbb{R}^{l_a}$ & Latent features of a candidate item $a$ \\ 
        $p \in \mathcal{P}$ & Time period variable, takes values in the range $\mathcal{P} := \{0, 1, \dots, 34\}$, representing the 35 time periods within a week \\
        $\bm s_{p} \in \mathbb{R}^{d_p}$ & Time period embedding of time period $p$ \\ 
        $\bm \Theta_p^*$ & True user preference matrix at the time period $p$ \\
        \bottomrule 
    \end{tabularx}
\end{table}

\subsection{Time-Varying User Preferences}

\label{sec:timevaryinguserpreferences}
In this section, we formally describe the time-varying user preferences mentioned in the introduction. 


As detailed in Table~\ref{tab:notations}, we introduce a time period variable $p$ to capture specific time patterns, such as hours within a day and days across the week. In particular, each week is segmented into seven days (Monday to Sunday), with each day further partitioned into five distinct intervals: morning (8:00 AM to 11:30 AM), noon (11:30 AM to 2:00 PM), afternoon (2:00 PM to 5:30 PM), night (5:30 PM to 10:00 PM), and the remaining period. Consequently, the time period variable $p$ takes 35 discrete values ranging sequentially from 0 to 34. Each value of $p$ is then encoded into a corresponding \emph{time period embedding}, represented as $\bm s_p \in \mathbb{R}^{d_p}$.


Traditional recommender systems typically assume a stationary environment, where user feedback mechanisms remain consistent over time. Under this assumption, the \emph{reward generation probability}, defined as $\operatorname{Pr}\{r(u,a)=1 \mid \bm c_u,\bm c_a\}$, is invariant across all time steps $t \in [T]$, suggesting that user $u$'s preference for item $a$ is independent of recommendation timing. However, real-world streaming recommender systems often exhibit periodic fluctuations in user preferences, as demonstrated by \cite{gao2013exploring}. For instance, users might favor office-related items on weekday mornings and entertainment venues like bars at night, leading to temporally distinct feedback patterns. Therefore, considering context $\bm c_u,\bm c_a$, time period variable $p$, and its corresponding embedding $\bm s_p$, the following inequality generally holds:
\begin{equation}
\operatorname{Pr}\{r(u,a)=1 \mid \bm c_u,\bm c_a,\bm s_{p=i}\}\neq\operatorname{Pr}\{r(u,a)=1 \mid \bm c_u,\bm c_a,\bm s_{p=j}\},
\end{equation}
where $i,j \in \{0,\dots,34\}$ and $i \neq j$. Here, $\operatorname{Pr}\{r=1\mid\bm c_u,\bm c_a,\bm s_p\}$ denotes the \emph{time-varying reward generation probability}, quantifying user $u$'s preference for item $a$ at a specific time period $p$.

Formally, we can represent the observed reward generated by the time-varying reward generation probability as $r(u, a, p)$. 
Given a user $u$, an item $a$, and a time period $p$, the generation process of the observed reward can be formalized as follows:
\begin{equation}
\label{eq:hyperbandit:reward:eta}
r (u, a, p)  := r^* (u, a, p) + \eta_{p},
\end{equation}
where $r^* (u, a, p)$ represents the \emph{true reward}, and $\eta_p$ is a random variable under time period $p$ drawn from a distribution with zero mean. This additional error term $\eta_p$ captures the noise or uncertainty present in the observations.
Clearly, we have the following expected reward: $$\mathbb{E} [r (u, a, p)] = r^* (u, a, p) = \operatorname{Pr}\left\{r (u, a, p)=1 \mid \bm c_u, \bm c_a, \bm s_{p}\right\}.$$

Next, we make specific assumptions about the form of the expected reward $\mathbb{E} [r (u, a, p)]$, i.e., the true reward. One straightforward approach is to concatenate the time period embedding with the context feature vectors. However, existing research~\cite{galanti2020modularity} has shown that directly concatenating features from different spaces can make it difficult to capture meaningful information (i.e., time-varying information in this paper). 
To address this issue, we extend the existing linear expected reward in the contextual bandit setting by introducing the true user preference matrix $\bm \Theta_p^*$.
Then, we can specify the true reward as the following \emph{time-varying true reward}:
\begin{equation}
\label{eq:hb:time_reward}
r^*(u, a, p) := \bm c_a^\intercal \bm \Theta_p^* \bm{c}_u, 
\end{equation}
where the \emph{true user preference matrix} $\bm \Theta_p^* \in \mathbb{R}^{d_a \times d_u}$ is utilized to map the user contexts through a linear mapping, taking into account the time period $p$ as well as its embedding $\bm s_p$ as conditions. By applying this mapping, the resulting vector $\bm \Theta_p^* \bm c_u$ can effectively capture the time-varying user preferences across different time periods.
In particular, when $d_a = d_u$ and $\bm \Theta_p^*$ is the identity matrix, the time-varying true reward degenerates to a fixed true reward in the traditional contextual bandit setting.

\section{\mbox{HyperBandit+: The Proposed Algorithm}}
\label{sec:method}
In this section, we introduce a novel bandit algorithm called HyperBandit+. This algorithm captures the relationship between time and user preference using a hypernetwork and improves performance through the integration of LLMs and Euler embedding.  We first present the framework of the whole algorithm. Then we demonstrate how the LLMs and Euler embedding enhances the algorithms and how the hypernetwork works in bandit policy with details. Lastly, we introduce an efficient training method with low-rank factorization to accelerate the training procedure.




\subsection{Algorithm Overview}

\begin{figure}[!t]
    \centering
    \includegraphics[width=\linewidth]{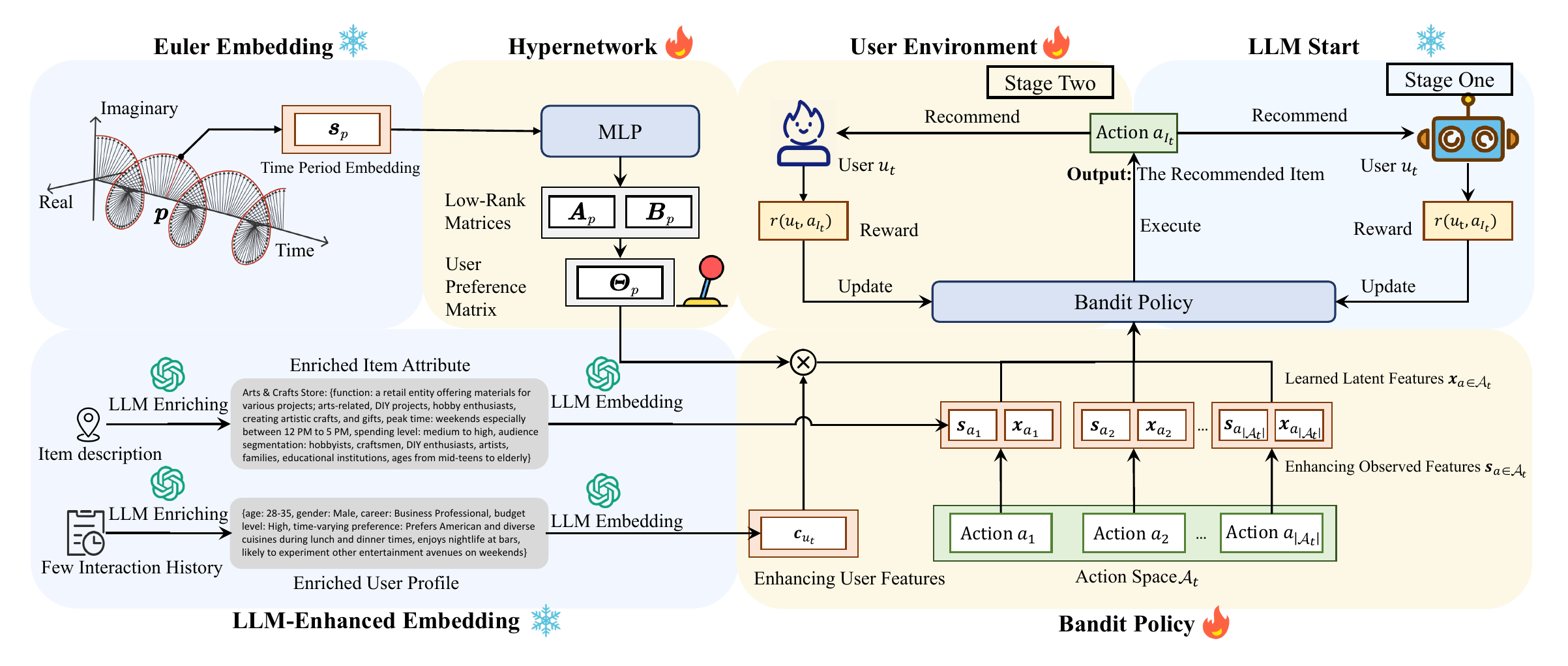}
    \caption{The structure of HyperBandit+ (HyperBandit with extensions of LLM Start, Euler Embedding and LLM-Enhanced Embedding at time $t$). The workflow involves the bandit policy selecting an arm (recommending an item) from the candidate pool, interacting with the environment to obtain feedback, and subsequently updating the bandit policy. ``Stage One'' and ``Stage Two'' represent the two environments that interact with the bandit policy sequentially. The symbol \includegraphics[width=0.025\linewidth]{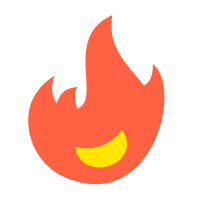} indicates the availability of the module during the online stage, while \includegraphics[width=0.024\linewidth]{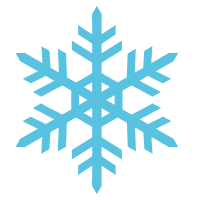} represents the offline policy.}
    \label{fig:HyperBandit+_structure}
\end{figure}




Figure~\ref{fig:HyperBandit+_structure} illustrates the structure of HyperBandit+. The algorithm can be divided into two phases: the online phase and the offline phase, represented in the figure by distinct colors indicating the work time of each module.

The offline phase consists of three modules. \textbf{Euler Embedding} module computes Euler embeddings for various time periods during this phase. In the online phase, these embeddings are directly queried based on the current time period. \textbf{LLM-Enhanced Embedding} module completes a two-step embedding generation during the offline phase. Initially, it infers user profiles and enriches item attributes based on the initial user interaction history, completing the side information. Subsequently, it utilizes this information to encode user-item interactions, yielding LLM-enhanced embeddings for use in the online phase. \textbf{LLM Start} operates before the online phase. Specifically, we proactively employ the LLM to simulate interaction data with feedback. By pre-learning from this simulated data, the bandit policy can achieve improved performance during the initial online stages, mitigating the delay associated with LLM inference.


The online phase consists of two satges: \textbf{(1) Inference Stage.} Given the current time period as input, the Euler embedding module initially maps it to the complex space. The hypernetwork then generates a user preference matrix that maps user context features to a time-aware preference space. Here, the user context features refer to the LLM-enhanced embedding. Subsequently, in conjunction with the enhanced item embedding, the bandit policy utilizes the mapped user context features to recommend a suitable item to the current user. \textbf{(2) Train Stage.} The bandit policy utilizes user feedback to obtain a closed-form solution through ridge regression, thereby updating the internal parameters. The hypernetwork is trained in a mini-batch manner to adaptively adjust the user preference matrix in the bandit policy for a given time period. The algorithm's detailed procedure is outlined in Algorithm \ref{alg:hyperbandit}. It's important to note that the user preference matrix is estimated using two low-rank matrices during training. This specific training technique will be discussed in detail in Sec.~\ref{sec:efficienttraining}.



\subsection{Data Augmentation via LLMs}

\subsubsection{Recommendation with Side Information}
Due to the interaction sparsity and the low-quality challenges, we often get a poor performance in some online learning algorithms during the initial stage. To handle this problem, many efforts introduced side information (i.e., the user profiles and the item attributes) to assist recommendation. 
In the MAB setting, when the policy $\pi$ is fixed, each action selection is solely related to the context features of the user and item (that is $\bm{c}_u$ and $\bm{c}_a$). Therefore, the key challenge in our MAB task is how to utilize interaction data $\mathcal{E}$ to learn improved context features (denoted as $\bm{c}_u$ and $\bm{c}_a$, collectively represented as $\textbf{E}$). Actually, the learning process of MAB is to have a better representation of $\textbf{E}$, incorporates side information $\textbf{F}$ and is formulated as maximizing the posterior estimator $p(\textbf{E} ~| ~\textbf{F}, \mathcal{E})$: 
\begin{align}
    \label{side information}
    \textbf{E}^* = \argmax\limits_{\textbf{E}} p( \textbf{E} ~| ~\textbf{F}, \mathcal{E}) 
\end{align}
where $p( \textbf{E} ~| ~\textbf{F}, \mathcal{E})$ means to encode as much user-item interaction information from $\mathcal{E}$ and side information $\textbf{F}$ as possible for accurate arm selection $a_{I_t} =\arg \max _{a \in \mathcal{A}_{t}} f_t\left(\bm c_{u} , \bm c_a \right)$ as defined in Sec~\ref{3.1}.



To enhance context features effectively, we employ Large Language Models (LLMs). In our algorithm, the item context feature $\bm{c_a}(t)$ at time $t$ for item $a \in \mathcal{A}_t$ consists of two components: the latent features $\bm{x}_a(t)$, which require online learning, and the observed features $\bm{s}_a$, which, like the user context feature $\bm{c}_u$, remain fixed during the online learning stage.
Consequently, we have designed two approaches for data augmentation targeting these two feature components. One is \textbf{LLM-Enhanced Embedding}, aiming to enhance the observed features $\bm{s}_a$ and user context feature $\bm{c_u}$. The other is \textbf{LLM Start}, aiming to improve the latent features $\bm{x}_a(t)$).

\begin{figure}[!t]
    \centering
    \includegraphics[width=\linewidth]{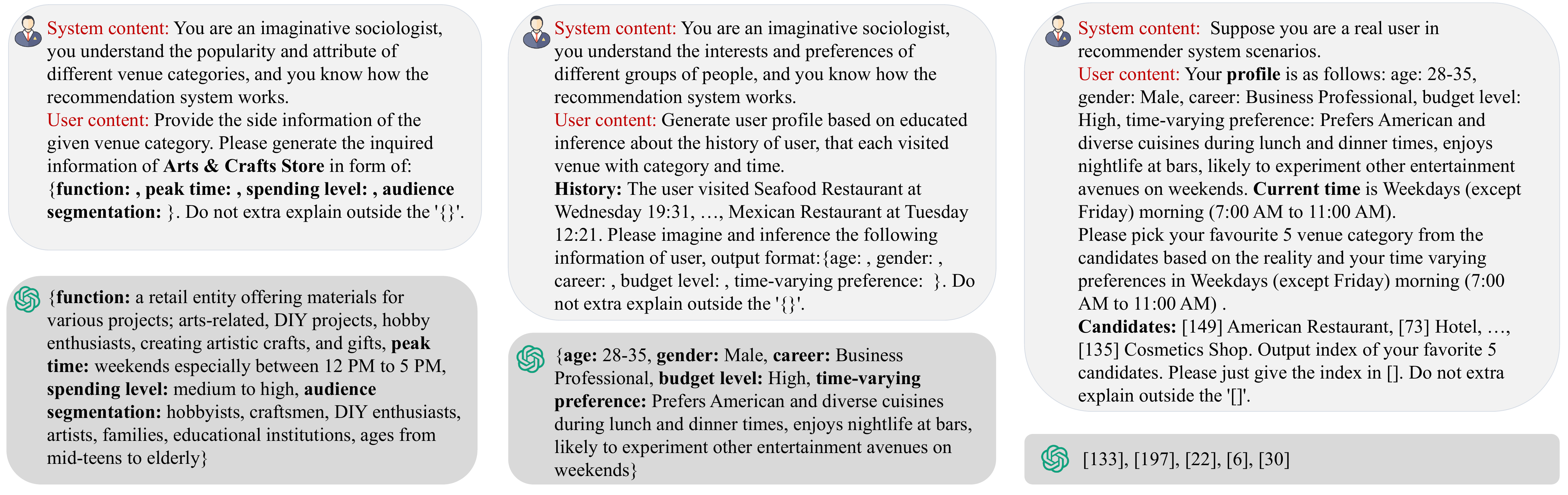}
    \subcaptionbox{Item attribute generation.}[0.32\linewidth]{\hfill}
    \subcaptionbox{User profile generation.}[0.32\linewidth]{\hfill}
    \subcaptionbox{Simulated data generation.}[0.32\linewidth]{\hfill}
    
    \caption{Constructed prompts for LLMs'  side information enhancing  and data simulation. The light gray box contains role assignment, task description, specific information, and output format description. The dark gray box contains the corresponding format for the response.}
    \label{fig:prompt}
    
\end{figure}

\subsubsection{LLM-Enhanced Embedding}
To enhance observed features, we utilize the knowledge base and reasoning capabilities of Large Language Models (LLMs) to summarize user profiles. This involves leveraging users' initial historical interactions and item information, mitigating privacy limitations. Moreover, the LLM-based generation of item attributes aims to create space-unified and informative attributes. Our paradigm for LLM-based side information augmentation comprises two steps:
\begin{enumerate}[(1)]
    \item User/Item Information Augmentation. Using prompts derived from the dataset’s few initial interaction history, we enable LLM to infer user profiles including age, gender, career, budget level and time-varying preference that are not originally part of the dataset. Since the items contains the category feature, we directly enrich the item attributes through the semantic information, thus including function, spending level, audience segmentation. Specific examples are shown in Fig.~\ref{fig:prompt} (a) (b).
    \item LLM Encoding. Utilizing LLM as an encoder, we encode augmented item information as features, which then serve as input for the downstream bandit policy. LLM's efficient and state-of-the-art language understanding capabilities enable the debiasing of item attributes. To align the user feature space with the item feature space, we utilize the average vector of a user’s interacted items in the first week, which helps maintain consistency with the dot-product correlation used in the bandit policy.
\end{enumerate}

\subsubsection{LLM Start}


As mentioned in the preceding section, we obtained augmented side information for users/items and further acquired LLM-Enhanced static embeddings (i.e., $\bm{s}_a$ and $\bm{c_u}$). In the bandit policy, these two components remain unchanged, representing the inherent information of users and items. In addition to them, the bandit policy in this algorithm typically requires learning the latent feature $\bm{x}_a(t)$, which is determined by parameters $\bm \Phi_{a, t}$, $\bm b_{a, t}$. These parameters are incrementally updated after each time step $t$ based on feedback using a closed-form solution, a process that typically requires numerous time steps to achieve algorithm convergence. Based on the augmented side information, we use LLM as a user simulator to generate extensive interaction data facilitating the algorithm to warm start during the offline stage. Specifically, we feed LLM with each user's augmented side information, historical interactions, time information, and item candidates pool. Here, following the setting in \cite{wei2023llmrec}, we introduce item candidates pool to reduce the cost associated with ranking all items. The candidates is selected by the base recommender such as BPR \cite{rendle2012bpr}.  Subsequently, LLM is instructed to select the $top-k$ items most likely to be interacted with the specific user from the candidates at the given time, which are considered as positive samples for the down stream bandit policy. The prompt is shown in Fig.~\ref{fig:prompt} (c).
Through the above operations, we can obtain the simulated interaction data for each user across various time periods.




\subsection{Hypernetwork Assisted Bandit Policy}

\subsubsection{Bandit Policy using User Preference Matrix}
To estimate the time-varying true
reward and account for the user preference shift in each time period, we propose a novel bandit policy that utilizes an estimate of the true user preference matrix $\bm \Theta_p^*$ in Eq.~\eqref{eq:hb:time_reward}. Specifically, the \emph{estimated user preference matrix}, denoted as $\bm \Theta_p \in \mathbb{R}^{d_a\times d_u}$, captures the changes in user preferences during time period $p$. We estimate $\bm \Theta_p$ using a hypernetwork, which will be introduced in Sec.~\ref{sec:HB:Hypernetwork}. 
The estimated user preference matrix allows us to adapt our bandit policy to the evolving user preferences. 
Formally, 
assuming that time step $t$ belongs to time period $p$, 
given a user context $\bm c_u \in \mathbb{R}^{d_u}$ and the estimated user preference matrix, 
the following ridge regression over the current interaction history is employed to estimate the item context $\bm c_{a} (t)$ at time $t\in [T]$:
\begin{equation}
 \label{eq: regularized quadratic loss}
\begin{aligned}
               \bm c_{a} (t) 
                =
                \argmin \limits_{\bm c_{a} \in \mathbb{R}^{d_a}}
                \sum_{(u, a, r) \in \mathcal{H}_t}
                \left[\bm c_{a}^{\intercal}  \bm \Theta_p \bm c_{u} - r(u,a,p) \right]^2 + \lambda \| \bm c_{a} \|_2^2,
\end{aligned}
\end{equation}
where $\hat{r}_{u, a, p} := \bm c_{a}^{\intercal}  \bm \Theta_p \bm c_{u}$ denotes the \emph{estimated time-varying reward}, $\mathcal{H}_t := \left\{ (u_k, a_{I_k}, r_k) \right\}_{k \in [t]}$ represents the \emph{interaction history} up to time $t$, $(u_k, a_{I_k}, r_k)$ denotes that the policy recommended item $a_{I_k}$ to user $u_k$ at time $k$ and received a reward $r_k$, and $\lambda >0$ is the regularization parameter.

To reduce the uncertainty of user preference estimations, we introduce the observed item features. Specifically, we split the context $\bm c_a (t)$ at time $t$ of item $a \in \mathcal{A}_t$ into two parts, represented as $\bm c_a (t):= [\bm s_a^\intercal, \bm x_a (t)^\intercal]^\intercal \in \mathbb{R}^{d_a}$, which includes: the observed features $\bm{s}_a \in \mathbb{R}^{o_a}$, and the latent features $\bm{x}_a(t) \in \mathbb{R}^{l_a}$ that needs to be learned online, where $d_a = o_a + l_a$. Accordingly, we redefine the estimated user preference matrix as $ \bm \Theta_p= \left[\bm \Theta_p^{s \intercal}, \bm \Theta_p^{x \intercal} \right]^\intercal$, where $\bm \Theta_p^{s} \in \mathbb{R}^{o_a\times d_u}$ corresponds to the observed item features $\bm s_a$, and $\bm \Theta_p^{x} \in \mathbb{R}^{l_a\times d_u}$ corresponds to the latent item features $\bm x_a(t)$.
As a result, we can rewrite the ridge regression  in Eq.~\eqref{eq: regularized quadratic loss} as follows:
\begin{equation}
\label{eq:regularized_quadratic_loss_enhanced}
               \bm x_{a} (t) =
                \argmin \limits_{\bm x_{a} \in \mathbb{R}^{l_a}}
                \sum_{(u, a, r) \in \mathcal{H}_t}
                \left\{[\bm s_a^\intercal, \bm x_a(t)^\intercal]  \bm \Theta_p \bm c_{u} - r(u,a,p) \right\}^2 + \lambda \| \bm x_{a} \|_2^2.
\end{equation}

Next, we will derive the optimal solution of Eq.~\eqref{eq:regularized_quadratic_loss_enhanced} in Theorem~\ref{Hyperbandit:thm:solve:x_a}.
\begin{theorem}[Closed-form Solution of Eq.~\eqref{eq:regularized_quadratic_loss_enhanced}]
\label{Hyperbandit:thm:solve:x_a}
The ridge regression problem presented in Eq.~\eqref{eq:regularized_quadratic_loss_enhanced} admits a closed-form solution, which is given as follows:
\begin{equation*}
\begin{aligned}
\bm x_{a}(t) &= \left( \bm \Psi_{a,t} \right)^{-1} \bm b_{a,t}, \\
\bm \Psi_{a,t} &= \sum_{u \in \mathcal{U}_{a,t}}
\left(\bm \Theta_{p}^x \bm c_u\right)\left(\bm \Theta_{p}^x \bm c_u\right)^\intercal
+ \lambda \bm I_{l_a}, \\
\bm b_{a,t} &= \sum_{(u,a,r)\in\mathcal{H}_t}
\left(\bm \Theta_{p}^x \bm c_u\right)
\Bigl[
r(u,a,p) - \left(\bm \Theta_{p}^s \bm c_u\right)^\intercal \bm s_a
\Bigr].
\end{aligned}
\end{equation*}
where $\mathcal{U}_{a, t}$ denotes the set of users (possibly with duplicates) who have been recommended item $a$ until time $t$, and $\bm I_{l_a} \in \mathbb{R}^{l_a \times l_a}$ is a identity matrix. The statistics $\left(\bm \Psi_{a,t}, \bm b_{a,t}\right)$ can be updated incrementally and the detailed computation can be found in Algorithm~\ref{alg:hyperbandit}.
\end{theorem}
\begin{proof}[Proof of Theorem~\ref{Hyperbandit:thm:solve:x_a}]
Firstly, Eq.~\eqref{eq:regularized_quadratic_loss_enhanced} can be rewritten as follows:
\begin{equation}
\label{eq:regularized_quadratic_loss_enhanced:re}
\begin{aligned}
               \bm x_{a} (t) 
               &=
                \argmin \limits_{\bm x_{a} \in \mathbb{R}^{l_a}}
                \sum_{(u, a, r) \in \mathcal{H}_t}
                \left\{[\bm s_a^\intercal, \bm x_a(t)^\intercal]  \bm \Theta_p \bm c_{u} - r(u,a,p) \right\}^2 + \lambda \| \bm x_{a} \|_2^2
                \\
                &=
                \argmin \limits_{\bm x_{a} \in \mathbb{R}^{l_a}}
                \sum_{(u, a, r) \in \mathcal{H}_t}
                \left[\bm s_a^\intercal \bm \Theta_p^s\bm{c}_u +\bm x_a(t)^\intercal \bm \Theta_p^x \bm{c}_u - r(u,a,p) \right]^2 +  \lambda \| \bm x_{a} \|_2^2.
\end{aligned}
\end{equation}
Then, taking the derivative of Eq.~\eqref{eq:regularized_quadratic_loss_enhanced:re} and setting it equal to zero, we obtain:
\begin{equation*}
2\sum_{(u,a,r)\in\mathcal{H}_{t}}\left(\bm{s}_{a}^{\intercal}\bm{\Theta}_{p}^{s}\bm{c}_{u}\right)\bm{\Theta}_{p}^{x}\bm{c}_{u}+2\sum_{(u,a,r)\in\mathcal{H}_{t}}\left(\bm{x}_{a}(t)^{\intercal}\bm{\Theta}_{p}^{x}\bm{c}_{u}\right)\bm{\Theta}_{p}^{x}\bm{c}_{u}-2\sum_{(u,a,r)\in\mathcal{H}_{t}}r(u,a,p)\bm{\Theta}_{p}^{x}\bm{c}_{u}+2\lambda\bm{x}_{a}(t)=0.
\end{equation*}
We can rewrite the above equation in matrix form:
\[
\bm{\Psi}_{a,t}\bm{x}_{a}(t)=\bm{b}_{a,t}.
\]
Since $\bm{\Psi}_{a,t}$ is invertible, we can obtain the solution:
\[
\bm{x}_{a}(t)=\left(\bm{\Psi}_{a,t}\right)^{-1}\bm{b}_{a,t}.
\]
This completes the proof of Theorem~\ref{Hyperbandit:thm:solve:x_a}. 
\end{proof}




According to the UCB policy in bandit algorithms \cite{li2010contextual, wang2017factorization,Zhang2021Counterfactual}, we define the following UCB-based relevance score function for executing action (i.e., online recommendation) at time $t$:  
\begin{equation*}
\begin{aligned}
f_t (\bm c_u, \bm c_a(t)):= \left[ \bm s_{a}^{\intercal}, \bm x_{a}(t)^{\intercal} \right] \bm \Theta_{p} \bm c_{u} 
            +
            \alpha \left[\left(\bm \Theta_{p}^{x} \bm c_{u} \right)^\intercal \left(\bm \Psi_{a, t} \right)^{-1} \bm \Theta_{p}^{x} \bm c_{u} \right]^{\frac{1}{2}},   
\end{aligned}
\end{equation*}
where $\alpha > 0$ is the exploration parameter, and the term multiplied by $\alpha$ is the exploration term. 
In this way, the executed action at time $t$ can be selected by $a_{I_t} = \argmax_{a \in \mathcal{A}_t} f_t (\bm c_u, \bm c_a (t))$.  Note that the exploration term is related to $\Theta_{p}^{x}$ but independent of $\Theta_{p}^{s}$, where $\Theta_{p}^{x}$ denotes the part of $\bm{\Theta}$ corresponding to the latent variable $\bm{x}_a(t)$. This is because, in our model design, the observed features are static information generated by the LLM and remain fixed during online learning, 
whereas the latent features are dynamically learned from users' real-time feedback. The exploration term in UCB quantifies the model's uncertainty in reward estimation, which mainly arises from parameters being updated online (i.e., the latent features). Hence, the algorithm explores to reduce this uncertainty, and no exploration is needed for fixed observed features.
\begin{algorithm}[t]
    \caption{HyperBandit+}   
    \label{alg:hyperbandit}
    \begin{algorithmic}[1]
    \REQUIRE   Latent features of items $\bm x_{a\in \mathcal{A}} = \bm{0}^{l_a}$,data buffer $\mathcal{D}_{n=1}= \emptyset$,
    $\bm \Phi_{a \in \mathcal{A}, t=1} = \bm{O}^{l_a\times l_a}$, $\bm b_{a \in \mathcal{A}, t=1} =\bm{0}^{l_a}$,
    $\{T_n\}_{n \in [N]}$ set of time steps in each updating part of online stage, regularization parameter $\lambda > 0$, exploration parameter $\alpha > 0$.  
    
    \STATE Initialize hypernetwork parameters $\boldsymbol{\xi}_{n=1}$ with Xavier Normal
    \STATE  $// ~~ \texttt{Adopt Euler Embedding}$ 
    \STATE Compute the time period embedding $\{\bm s_p\}_{p \in \mathcal{P}}$
    \STATE  $// ~~ \texttt{Adopt LLM-Enhanced Embedding}$
    \STATE Generate the side information and the LLM-enhanced embedding $\{\bm c_u\}$ and $\{\bm s_a\}$   
    \STATE  $// ~~ \texttt{Adopt \hyperref[alg:LLMStart]{LLM Start} to initial parameters}$
    
    \STATE $\bm x_{a\in \mathcal{A}}$,
    $\bm \Phi_{a \in \mathcal{A}, t=1}$, $\bm b_{a \in \mathcal{A}, t=1}$, $\boldsymbol{\xi}_{n=1}$ = $\bm{LLM Start}$($\bm x_{a\in \mathcal{A}}$,
    $\bm \Phi_{a \in \mathcal{A}, t=1}$, $\bm b_{a \in \mathcal{A}, t=1}$, $\boldsymbol{\xi}_{n=1}$)
    \STATE  $//~~N~\texttt{Mini-batch for Online Stage}$
    \FOR{$n  \in [N]$}
    \STATE  $// ~~ \texttt{Bandit Policy Procedure}$
    
    \FOR{$t = 1 $ to $ T_n $}
            \STATE Receive the user $u_t$ and the time period embedding $\bm s_{p_t}$
            \STATE Obtain the set of candidate items $\mathcal{A}_t$
            \STATE Obtain the observed features $\bm s_a, \forall a \in \mathcal{A}_{t}$
            \STATE Obtain the latent features $\bm x_a(t), \forall a \in \mathcal{A}_{t}$
            \STATE Estimated the user preference matrix $\bm \Theta_{p_t}^{(n)} :=\left[ \bm \Theta_{p_t}^{s(n)\intercal}, \bm \Theta_{p_t}^{x(n)\intercal} \right]^\intercal \leftarrow h_{\bm \xi_{n}} (\bm s_{p_t})$
            \STATE Recommend item $a_{I_t} \in \mathcal{A}_{t}$ to user $u_t$ following $a_{I_t} \leftarrow \arg \max _{a \in  \mathcal{A}_t}
            \left[ \bm s_{a}^{\intercal}, \bm x_{a}(t)^{\intercal} \right] \bm \Theta_{p_t}^{(n)} \bm c_{u_t} 
            +
            \alpha \left[\left( \bm \Theta_{p_t}^{x(n)} \bm c_{u_t} \right)^\intercal \left(\bm \Psi_{a, t} \right)^{-1} \bm \Theta_{p_t}^{x(n)} \bm c_{u_t} \right]^{\frac{1}{2}}$ 
            
            \STATE Observe reward $r_t = r(u_t,a_{I_t},p_t)$
            \STATE $\mathcal{D}_{n} \leftarrow \mathcal{D}_{n} \cup \left\{ (u_t, a_{I_t}, p_t, r_t,\mathcal{A}_t)  \right\} $
            \STATE  $// ~~ \texttt{Bandit Policy Updating}$            
            \STATE Get the user preference vector $\bm P_{t} \leftarrow \left(\bm \Theta_{p_t}^{x(n)} \bm c_{u_t}\right)^\intercal  \in \mathbb{R}^{l_a}~$ for the latent item features

            \STATE Get the user preference vector $\bm Q_{t} \leftarrow \left(\bm \Theta_{p_t}^{s(n)} \bm c_{u_t}\right)^\intercal \in \mathbb{R}^{o_a}$ for the observed item features
     
            \STATE $\bm \Phi_{a_{I_t},t+1} \leftarrow \bm \Phi_{a_{I_t},t} + \bm P_{t}^\intercal \bm P_{t} $, \quad  $\bm \Psi_{a_{I_t},t+1} \leftarrow \lambda \bm I + \bm \Phi_{a_{I_t},t+1}$
            
            \STATE $\bm b_{a_{I_t},t+1} \leftarrow \bm b_{a_{I_t},t} + \bm P_{t}^\intercal 
            \left(
                \bm r_t - \bm Q_{t} \bm s_{a_{I_t}}
            \right)$
            
            \STATE $\bm x_{a_{I_t},t+1} \leftarrow 
                \left( \bm \Psi_{a_{I_t},t+1}  \right)^{-1} ~\bm b_{a_{I_t},t+1}$
            \ENDFOR
            
            \STATE  $//~~ \texttt{Hypernetwork Updating}$
        
            \STATE Update hypernetwork parameter $\bm \xi_{n+1} \leftarrow \Delta (\bm \xi_{n})$ using efficient training method via low-rank factorization (in Sec.~\ref{sec:HB:Lowrank:training}) and the Adam optimizer on $\mathcal{D}_{n}$
            \STATE Release $\mathcal{D}_{n}$
            and set $\mathcal{D}_{n+1}\leftarrow \emptyset $
    \ENDFOR
\end{algorithmic}
\end{algorithm}

\begin{algorithm}[t]
    \caption{LLM Start}   
    \label{alg:LLMStart}
    \begin{algorithmic}[1]
    \REQUIRE  $\bm x_{a,t}$,
    $\bm \Phi_{a ,t} $, $\bm b_{a ,t}$, $\boldsymbol{\xi}$,

    \STATE  $// ~~ \texttt{Adopt Euler Embedding}$ 
    \STATE Compute the time period embedding $\{\bm s_p\}_{p \in \mathcal{P}}$
    \STATE  $// ~~ \texttt{Adopt LLM-Enhanced Embedding}$
    \STATE Generate the side information and the LLM-enhanced embedding $\{\bm c_u\}$ and $\{\bm s_a\}$   
   \STATE  $//~~N^{\textit{LLM}}~\texttt{Mini-batch for Online Stage}$
    \FOR{$n  \in  [N^{\textit{LLM}}]$}
    \FOR{$t = 1 $ to $ T_n $}
            \STATE $// ~~ \texttt{Bandit Policy Procedure}$
            \STATE $...$
            \ENDFOR
          
            \STATE  $// ~~ \texttt{Hypernetwork Updating}$
            \STATE $...$
        
    \ENDFOR
    
    \ENSURE   $\bm x_{a,t}$,
    $\bm \Phi_{a ,t} $, $\bm b_{a ,t}$, $\boldsymbol{\xi}$
\end{algorithmic}
\end{algorithm}

\subsubsection{Hypernetwork for Time-Varying Preference}
\label{sec:HB:Hypernetwork}
In the last section, we describe the bandit policy given parameter matrix $\bm \Theta_p$ in time period $p$. In this section, we explain how the hypernetwork generates the parameter matrix $\bm \Theta_p$. 
The main concept involves utilizing a hypernetwork that takes the embedding of the current time period as input and generates the parameters of the user preference matrix in the bandit policy. This enables the policy to adapt and adjust itself to accommodate changes in the distribution of user preferences over time.


To ensure stability in online recommendation, we incrementally update the hypernetwork $h$ in mini-batches, where the total $T$ time steps are divided into $N$ parts, and the $n$-th part, $n \in [N]$, contains $T_n$ time steps, corresponding to $T_n$ interaction histories.
In this way, the hypernetwork $h$ is updated $N$ times, and during the $n$-th update, the data buffer $\mathcal{D}_{n}:= \left\{ (u_t, a_{I_t},p_t, r_t, \mathcal{A}_t) \right\}_{t \in [T_n]}$ is used as the training data\footnote{It is important to note that the interaction history corresponding to the same index $t$ in different data buffers $\mathcal{D}_n$ may be different.}, where $r_t = r(u_t,a_{I_t},p_t)$. Then, given the time period embedding $\bm s_p$, the hypernetwork after the $(n-1)$-th update, denoted by $h_{\bm \xi_n}$, can be represented by:
\begin{equation}
\label{eq:hn:formulation}
\bm \Theta_{p}^{(n)}  := 
h_{\bm \xi_{n}} (\bm s_p),
\end{equation}
where $\bm \xi_n$ represents the model parameters of the hypernetwork, and the superscript $(n)$ on $\bm \Theta_{p}^{(n)}$ indicates that it is generated by $h_{\bm \xi_n}$. 
As illustrated in Figure~\ref{fig:HyperBandit+_structure}, we implement the hypernetwork $h$ using a Multi-Layer Perceptron (MLP). 
In this way, the MLP acts like a condition network, inputting the embedding of current time period, outputting the corresponding user preference matrix. 
Besides, the time period embedding $\bm s_p$ is generated via GloVe model~\cite{pennington2014glove} by imputing the current time period $p$.



To train the hypernetwork, we take inspiration from the Listnet loss design~\cite{cao2007learning} to quantify the discrepancy between the estimated reward and the true label for each candidate item in $\mathcal{A}_t$ at time~$t$. 
Specially, we use the estimated time-varying reward $\hat{r}
_{u, a, p} = \bm c_{a}^{\intercal}  \bm \Theta_p \bm c_{u}$ in Eq.~\eqref{eq: regularized quadratic loss}, 
and construct the true labels according to the following rules: if the user $u$ clicks on the recommended item $a$ in time period $p$, the label is set to 1; if the user skips the recommended item, the label is set to -1; if the item is a candidate but not recommended, the label is set to 0. Formally, $y_{u, a, p} = 1$ if $r_{u, a, p} = 1$; $y_{u, a, p} = -1$ if $r_{u, a, p}$ = 0; $y_{u, a, p} = 0$ if it is a candidate item but not recommended. Then, assuming that the number of actions is $M := |\mathcal{A}_1| = \cdots = |\mathcal{A}_T|$, 
during the $n$-th incremental update, the loss function on the data buffer $\mathcal{D}_{n}$ could be shown as follows: 
\begin{equation*}  
\mathcal{L}_{\bm \xi}^{(n)} = -
\sum_{t=1}^{T_n} 
\sum_{k=1}^{M} P \left(y_{u_t, a_k, p_t}\right) \log \widehat{P} \left(\hat{r}_{u_t, a_k, p_t}\right),
\end{equation*}
where $\bm \xi$ represents the hypernetwork parameters that need to be optimized, $p_t$ corresponds to the time period where index $t$ in $\mathcal{D}_n$ is located, and
\begin{equation*}    \widehat{P} \left(\hat{r}_{u_t, a_k, p_t}\right)=\frac{\exp \left(\hat{r}_{u_t, a_k, p_t}\right)}{\sum_{i=1}^{M} \exp \left(\hat{r}_{u_t, a_i, p_t}\right)}, \quad 
{P} \left(y_{u_t, a_k, p_t}\right)=\frac{\exp \left(y_{u_t, a_k, p_t}\right)}{\sum_{i=1}^{M} \exp \left(y_{u_t, a_i, p_t}\right)}.
\end{equation*}


\subsubsection{Euler Embedding for Hypernetwork}

As discussed in the preceding section, the hypernetwork takes time embedding ($\bm{s_p}$) as input and outputs the corresponding parameter matrix, effectively functioning as a time index model. Consequently, the encoding approach for the time period variable ($\bm{p}$) significantly influences downstream bandit policy. Fourier feature mapping, introduced by Rahimi and Recht~\cite{rahimi2007random} to approximate stationary kernel functions using random Fourier features based on Bochner's theorem, has found widespread application. Extending this technique to periodic time embedding, we employ a Fourier feature mapping function, denoted as $\gamma$, to transform the input $\bm{p}$ before feeding it into the hypernetwork. The function $\gamma$ maps the input $\bm{p}$ to the surface of a higher-dimensional hypersphere using a set of sinusoids. The mapped coordinates in hypersphere is denoted as $\bm s_p$: 
\begin{equation}
\label{eqn:gamma}
    \bm s_p = \gamma(\bm p) = [
     \cos(2 \pi \mathrm{a_1}(\bm p)),
     \sin(2 \pi \mathrm{a_1}(\bm p) ), 
    \ldots, 
    \cos(2 \pi \mathrm{a_m}(\bm p) ),
     \sin(2 \pi \mathrm{a_m}(\bm p) )]^{\top},
\end{equation}
where $\bm{m}$ represents the number of cycles, which is set to 2 in our algorithms, signifying the weekly and daily cycles. The function $\mathrm{a_i}(\bm{p}) \in [0,1], i\in {1,\dots,\bm{m}}$ denotes the relative position of $\bm{p}$ within the corresponding cycle. For instance, consider a scenario where $\bm{p}$ is characterized by two cycles: a weekly cycle and a daily cycle (i.e., $\bm{m} = 2$). When $\bm{p}=7$ (i.e., Tuesday afternoon), we have $\mathrm{a_1}(\bm{p}) = 2/7$ and $\mathrm{a_2}(\bm{p}) = 3/5$. This is because ``Tuesday'' corresponds to index 2 within the weekly cycle of 7 days, and ``afternoon'' corresponds to index 3 within the daily cycle of 5 time blocks. Details regarding the segmentation of time blocks and their corresponding indices within each cycle are provided in Section~\ref{sec:timevaryinguserpreferences}.

\subsection{\mbox{Efficient Training via Low-Rank Factorization}}
\label{sec:efficienttraining}

\subsubsection{Analysis of Low-Rank Structure of User Preference Matrix}

\label{subsec: Low-rank Factorization}

\begin{figure*}[t]
\centering
\begin{subfigure}[b]{0.32\textwidth}
\includegraphics[width=\textwidth]{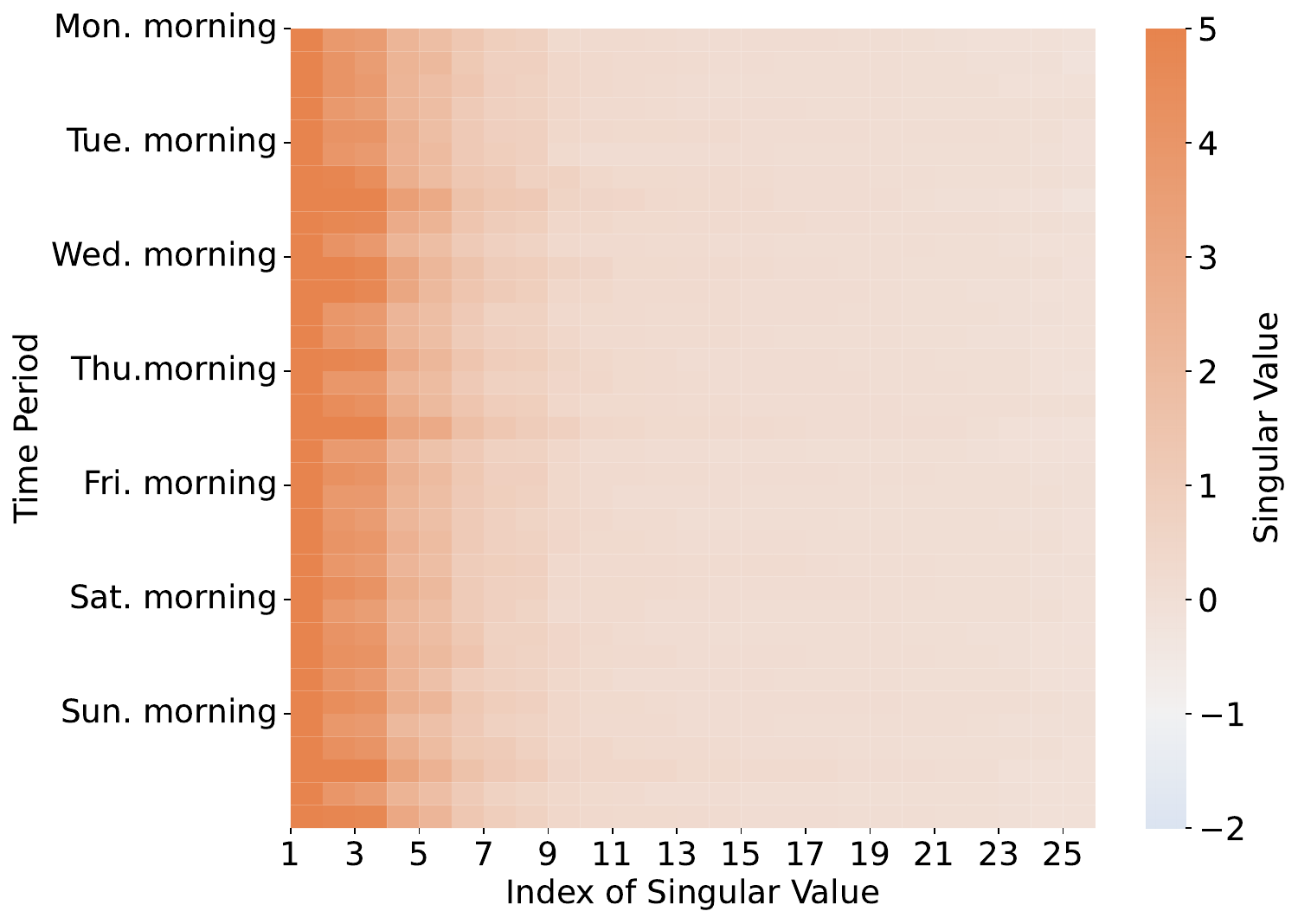}
\caption{KuaiRec}
\end{subfigure}
\hfill
\begin{subfigure}[b]{0.32\textwidth}
\includegraphics[width=\textwidth]{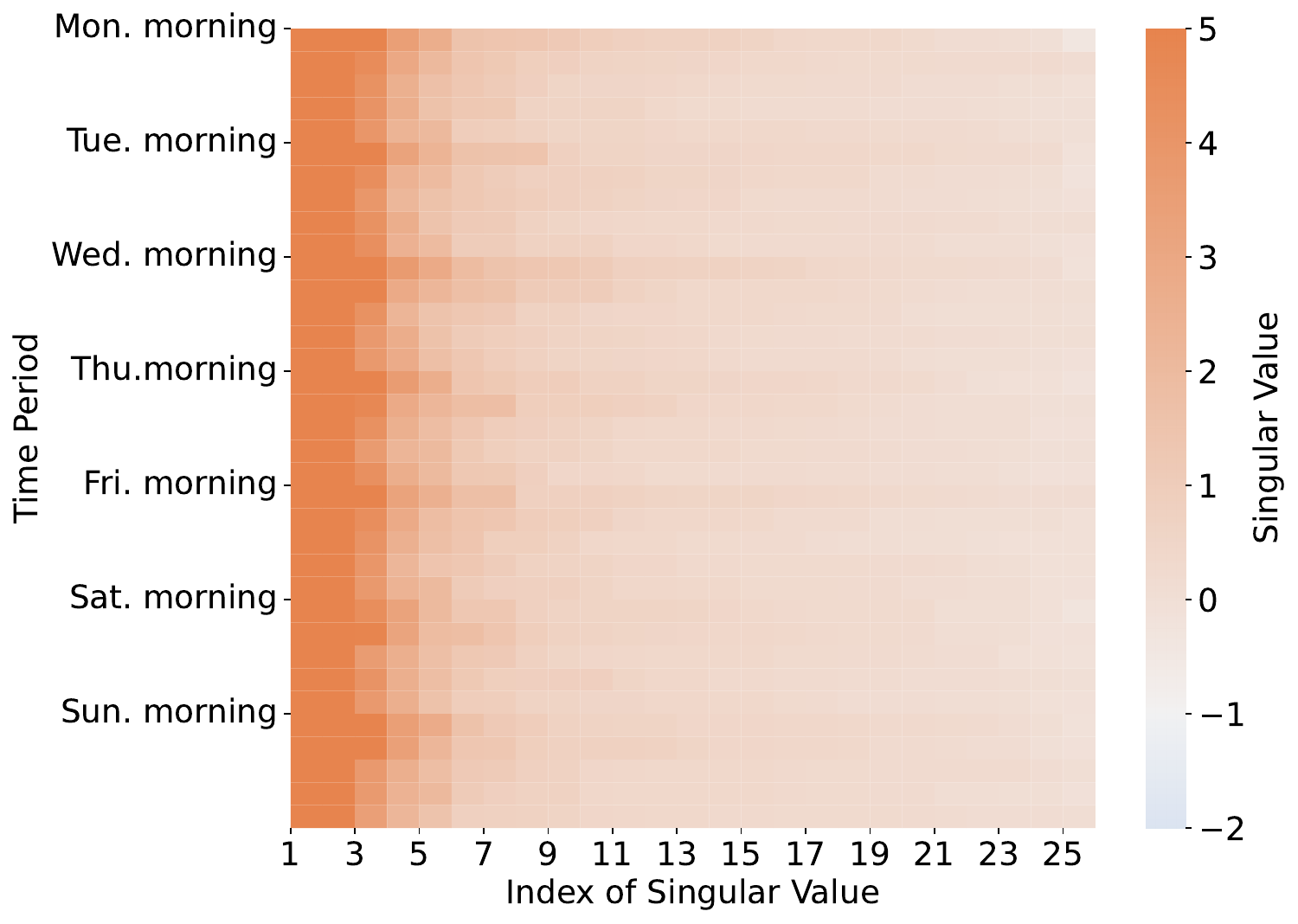}
\caption{NYC}
\end{subfigure}
\hfill
\begin{subfigure}[b]{0.32\textwidth}
\includegraphics[width=\textwidth]{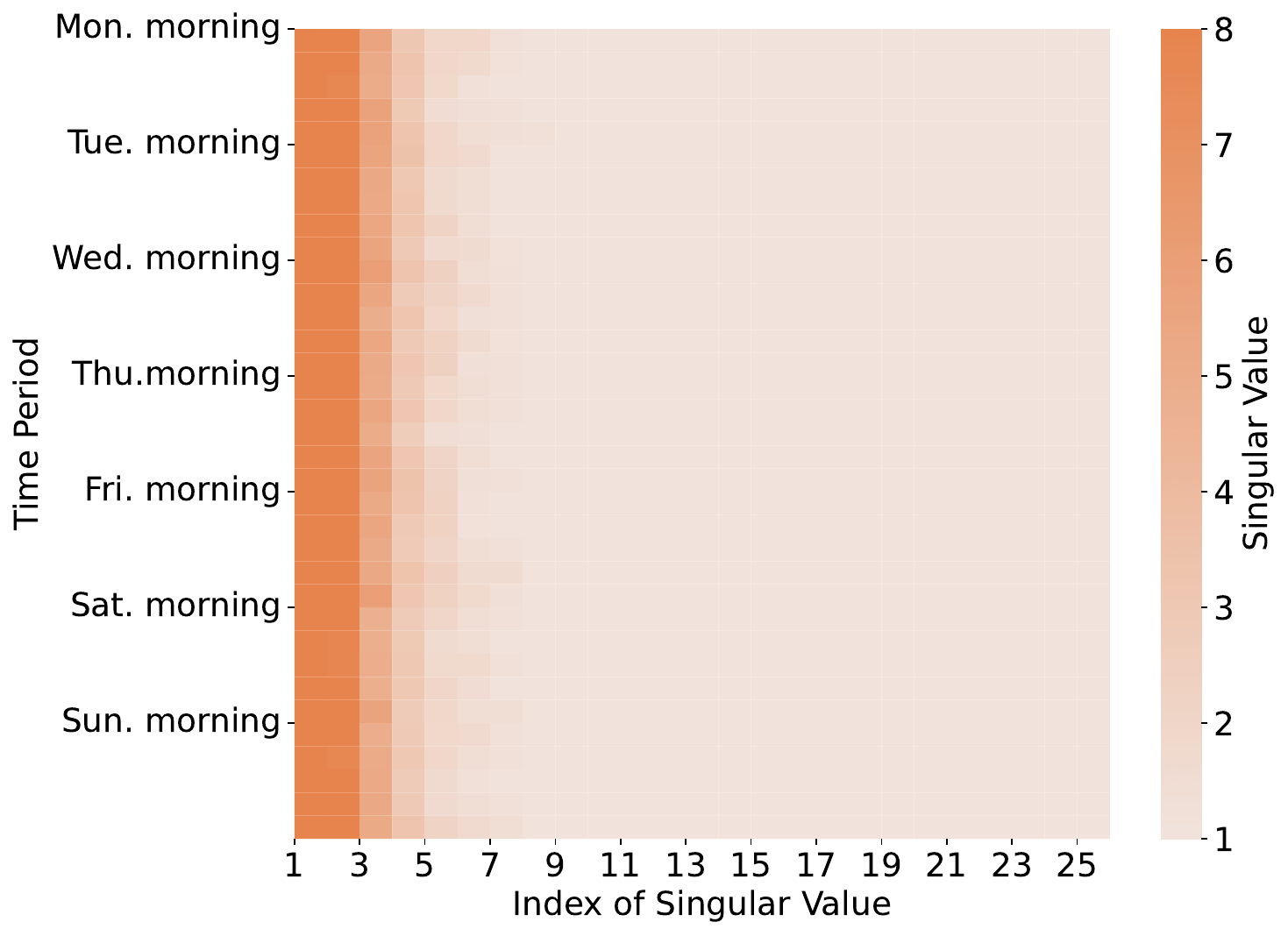}
\caption{TKY}
\end{subfigure}
\caption{
The distribution of singular eigenvalues (SEs) of user preference matrices across different time periods. The horizontal axis represents the index of SEs, arranged in descending order, while the vertical axis represents the time periods. The darkness of the colors corresponds to the magnitude of the singular values.}
\label{fig:SVD}
\end{figure*}
Since the user preference matrix $\bm \Theta_{p} \in \mathbb{R}^{d_a \times d_u}$ is generated by the hypernetwork in Eq.~\eqref{eq:hn:formulation}, a large output dimension (i.e., $d_a \times d_u$) 
would incur significant training costs. Hence, we consider representing the entire user preference matrix using a smaller number of parameters. 
Based on this motivation, it is natural to investigate whether the user preference matrix $\bm \Theta_p$ exhibits a low-rank structure. 
To verify the presence of low-rank structures, we perform singular value decomposition (SVD) on $\bm \Theta_{p}^{(N)}$ across different time periods. 
As shown in Fig.~\ref{fig:SVD}, when the singular values of the user preference matrices are sorted in descending order for different time periods, it becomes apparent that the distribution of singular values is concentrated in the first few dimensions, which is always less than half of the dimensionality of the user preference matrices\footnote{Following the setting of baselines~\cite{li2010contextual,wu2018learning,hariri2015adapting,wang2017factorization}, we set $d_a = d_u = 25$, thus $\bm\Theta_p$ is a square matrix. 
}. 
Based on this observation, we can conclude that there are strong low-rank structures present in the user preference matrices. This implies that a low-rank representation of the matrix $\bm \Theta_p$ could preserve nearly all of its information content.

\subsubsection{Training Process with Low-Rank Factorization}
\label{sec:HB:Lowrank:training}
Based on the analysis above, we try to improve the training efficiency of hypernetwork through explicitly modeling the low-rank structure in the user preference matrix $\bm \Theta_p$ (for ease of exposition, we omit the superscript of $\bm \Theta_p^{(n)}$ below). Specifically, we propose to approximate  $\bm\Theta_p$ with its low-rank approximation. Here, we leverage matrix factorization approach to achieve the approximation. 
Given the \emph{estimated rank} $\tau>0$, we model the low-rank structure of $\bm \Theta_p$ with the product of two rank-$\tau$ latent matrices $\bm A_p \in \mathbb{R}^{d_a\times \tau}$ and $\bm B_p\in \mathbb{R}^{d_u\times \tau}$, i.e., $\bm \Theta_p \approx \bm A_p  \bm B_p^\intercal$, as shown in Fig.~\ref{fig:lowrank_structure}. 



In the implementation of the hypernetwork, given a time period $p \in \mathcal{P}$, the hypernetwork outputs a vector represented by $\mathrm{Concat}(\mathrm{Vec}(\bm A_p),\mathrm{Vec}(\bm B_p) ) \in \mathbb{R}^{\tau d_a+\tau d_u}$, where $\mathrm{Concat}(\cdot)$ denotes the concatenation operation. Then, the vector $\mathrm{Vec}(\bm A_p) \in \mathbb{R}^{\tau d_a}$ is reshaped into a matrix  $\bm A_p \in \mathbb{R}^{d_a \times \tau}$, and the vector $\mathrm{Vec}(\bm B_p) \in \mathbb{R}^{\tau d_u}$ is reshaped into a matrix $\bm B_p \in \mathbb{R}^{d_u \times \tau}$. Finally, the product $\bm A_p  \bm B_p^\intercal$ is obtained to estimate $\bm \Theta_p$.  This matrix factorization reduces the output dimension of the hypernetwork $h$ (defined in Eq.~\eqref{eq:hn:formulation}) from $d_a  d_u$ to $\tau  (d_a + d_u)$, 
effectively alleviating the training efficiency issues.

\begin{figure}[t]
    \centering
    \includegraphics[width=0.8\linewidth]{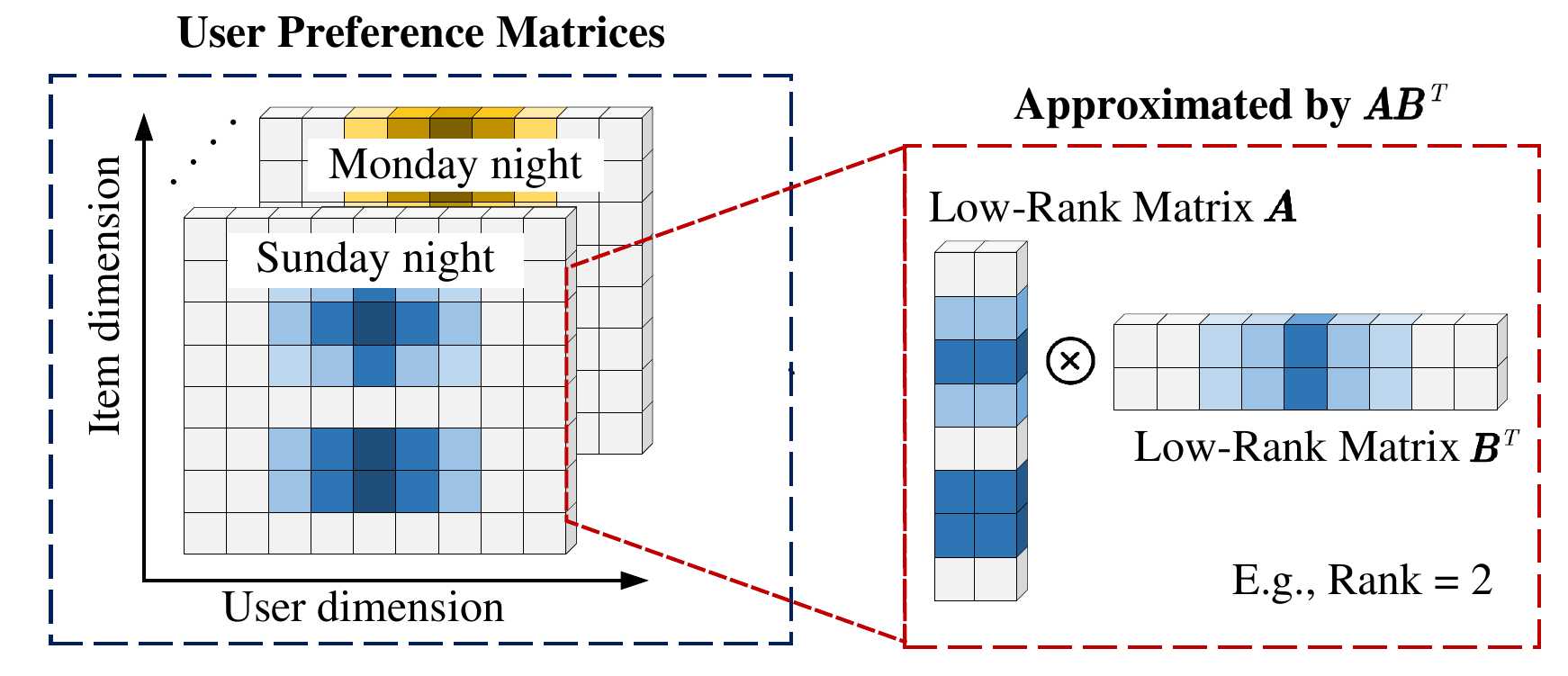}
    \caption{User preference matrix estimation using low rank factorization: An example with estimated rank $\tau = 2$.}
    \label{fig:lowrank_structure}
\end{figure}


\section{Dynamic Regret Analysis}

\subsection{Dynamic Regret Upper Bound}
The dynamic regret bound serves as a fundamental theoretical guarantee for online learning algorithms \cite{Cesa2006Prediction,Bubeck2012Regret,Shalev2011OLA,Hazan2016Introduction,Zhang2019Survey}. 
In this section, we focus on the dynamic regret that in non-stationary environments \cite{cheung2019learning,besbes2019optimal}, and provide a dynamic regret upper bound of the proposed HyperBandit. 
First, we define the \emph{dynamic regret} as follows:
\begin{equation}
\label{eq:HB:regret:def}
    \mathrm{D\text{-}Reg}(T) := 
    \sum_{t \in [T]} 
    \left[
        r^*(u_t, a_t^*, p_t) - r^*(u_t, a_{I_t}, p_t) 
    \right],
\end{equation}
where $a_t^*$ represents the action with the highest time-varying true reward $r^*$ (defined in Eq.~\eqref{eq:hb:time_reward}) at time $t$, $u_t$ denotes the user for whom the item is recommended at time $t$, and $p_t$ represents the time period to which $t$ belongs.  
Recalling that $I_t$ denotes the index of the action executed by HyperBandit at time $t$, the regret in Eq.~\eqref{eq:HB:regret:def} measures the difference between the accumulated time-varying true rewards of the best dynamic policy and our policy.


\begin{lemma}\label{lem:expect_real}

For any $t \in [T]$ and $\delta \in (0, 1)$, with probability at least $1 - \delta$, the following holds for any time step $t$.

\begin{equation}
    \left|\langle \bm P_t, \hat{\bm x}_{a}(t) - \bm x_{a}(t) \rangle \right| \leq L \sum_{p=1}^t \|\bm x_a(p) - \bm x_a(p+1)\|_2 + \beta_t \|\bm{\bm P_t}\|_{\bm \psi_{a, t-1}^{-1}},
\end{equation}
where $\bm{P}_t$ is the intermediate variable defined in Alg.~\ref{alg:hyperbandit}, produced using the real data at step $t-1$. Moreover, $\bm{\psi}_{a,t-1} = \lambda I_d + \bm{\phi}_{a,t-1} + \bm{\phi}_{a,\mathrm{LLM}}$, and $\bm{\phi}_{a,t-1}$ and $\bm{\phi}_{a,\mathrm{LLM}}$ also denote the intermediate variable $\bm{\phi}$ defined in Alg.~\ref{alg:hyperbandit}, generated from the real data at step $t-1$ and by the LLM-Start module, respectively.

\end{lemma}

\noindent \paragraph{Proof of Lemma~\ref{lem:expect_real}} From the model assumption and the estimator, we can verify that the estimate error can be decomposed as

\begin{equation}
    \hat{\bm x}_{a}(t) - \bm x_{a}(t) = \bm \psi_{a,t-1}^{-1} \left( \sum_{s=t_0}^{t-1} \bm P_s \bm P_s^\top (\bm x_{a}(s)- \bm x_{a}(t)) + \sum_{s=t_0}^{t-1} \eta_s \bm P_s - \lambda \bm x_{a}(t) \right).
\end{equation}

\noindent
therefore, by Cauchy-Schwartz inequality, we know that for any $\bm{x} \in \mathcal{X}$,
\begin{equation}
    \left| \bm P_t (\hat{\bm x}_{a}(t) - \bm x_{a}(t)) \right| \leq \|\bm P_t\|_2 \cdot A_t + \|\bm P_t\|_{\bm \psi_{t-1}^{-1}} \cdot B_t,
    \label{equ:inequ13}
\end{equation}
where
\begin{equation}
    A_t = \left\| \bm \psi_{a,t-1}^{-1} \left( \sum_{s=t_0}^{t-1} \bm P_s \bm P_s^\top (\bm x_{a}(s) - \bm x_{a}(t)) \right) \right\|_2,  B_t = \left\| \sum_{s=t_0}^{t-1} \eta_s \bm P_s - \lambda \bm x_{a}(t)\right\|_{\bm \psi_{a,t-1}^{-1}}.
\end{equation}

\noindent
These two terms can be bounded separately, as proved in ~\cite{zhao2020simple}. We
directly present the conclusion, where $A_t$ and $B_t$ can be upper bounded as follows.

\begin{equation}
   A_t \leq \sum_{p=t_0}^{t-1} \|\bm x_a(p) - \bm x_a(p+1)\|_2,
    B_t \leq \beta_t,
    \label{inequ15}
\end{equation}
where  $\beta_t = \sqrt{\lambda} S + R \sqrt{2 \log \frac{1}{\delta} + d \log \left( 1 + \frac{(t - t_0) L^2}{\lambda d} \right)}$ is the confidence radius.

\noindent
Based on the inequality~\ref{equ:inequ13} and inequality~\ref{inequ15}, we have
\begin{equation}
    \left|\langle \bm P_t, \hat{\bm x}_{a}(t) - \bm x_{a}(t) \rangle \right| \leq L \sum_{p=1}^t \|\bm x_a(p) - \bm x_a(p+1)\|_2 + \beta_t \|\bm{\bm P_t}\|_{\bm \psi_{a, t-1}^{-1}},
\end{equation}
which completes the proof.  \qed
\begin{lemma}
\label{lemma:LLM}
The variance term of our algorithm is not larger than that without the LLM augumentation~\cite{zhang2024reward}.  Formally, for any $t \in \mathcal{T}, a \in \mathcal{A}$:

\begin{equation}
\left[ \bm P_t^\top (\bm \psi_{a,t-1})^{-1} \bm P_t\right]^{\frac{1}{2}} 
\leq 
\left[ \bm P_t^\top (\bm \Gamma)^{-1} \bm P_t\right]^{\frac{1}{2}}.
\label{equ:lemma2}
\end{equation}
where $\bm{P}_t$ is the intermediate variable defined in Alg.~\ref{alg:hyperbandit}, produced using the real data at step $t-1$. Moreover, $\bm{\psi}_{a,t-1} = \lambda I_d + \bm{\phi}_{a,t-1} + \bm{\phi}_{a,\mathrm{LLM}}$, and $\bm{\phi}_{a,t-1}$ and $\bm{\phi}_{a,\mathrm{LLM}}$ also denote the intermediate variable $\bm{\phi}$ defined in Alg.~\ref{alg:hyperbandit}, generated from the real data at step $t-1$ and by the LLM-Start module, respectively. For notational convenience, we denote $\lambda I_d + \bm{\phi}_{a,t-1}$ by $\bm{\Gamma}$.
\end{lemma}

\noindent \paragraph{Proof of Lemma~\ref{lemma:LLM}.} 

Eq.\eqref{equ:lemma2} is equivalent to
\begin{equation}
\bm P_t^\top (\bm \psi_{a,t-1})^{-1} \bm P_t
\leq 
\bm P_t^\top (\bm \Gamma)^{-1} \bm P_t.
\end{equation} 

Considering $\bm \psi_{a,t-1} = \lambda I_d + \bm \phi_{a,t-1} + \bm \phi_{a, \text{LLM}} $, by the Sherman-Morrison-Woodbury formula, we have  
\begin{equation}
    \begin{aligned}
    (\bm \psi_{a,t-1})^{-1}=\left( \bm \Gamma + \bm{\phi}_{a,\text{LLM}}\right)^{-1} = \left( \bm \Gamma +  \bm S^\top \bm S \right)^{-1} = \bm \Gamma^{-1} -  \bm \Gamma^{-1} \bm S^\top \left( I_d +  \bm S \bm \Gamma^{-1} \bm S^\top \right)^{-1} \bm S \bm \Gamma^{-1},
    \end{aligned}
\end{equation}
yielding that Eq.\eqref{equ:lemma2} is equivalent to
\begin{equation}
    \bm P_t^\top \bm H \bm P_t \geq 0,
    \label{equ:22}
\end{equation}
where
\begin{equation}
    \bm H = \bm \Gamma^{-1} \bm S^\top \left( I_d + \bm S \bm \Gamma^{-1} \bm S^\top \right)^{-1} \bm S \bm \Gamma^{-1}.
\end{equation}

Let \( \bm S = \bm U_d \bm \Sigma^{1/2}_d \bm V_d^\top \) be the Singular Value Decomposition (SVD) of \( \bm S \). Note that \( \bm \phi_{a,
\text{LLM}} = \bm V_d \bm \Sigma_d \bm V_d^\top \). We can obtain that \( \bm \Gamma \) is a square symmetric positive semi-definite matrix, since \( \bm \Gamma \) can be decomposed into
\[
\bm \Gamma = \bm Q  \bm Q^\top,
\]
where \( \bm P \bm \Lambda \bm P^\top \) is the SVD of \( I_d + \bm S \bm \Gamma^{-1} \bm S^\top \) and
\[
\bm Q = \bm \Lambda^{-1/2} \bm P^\top \bm S \bm \Gamma^{-1}.
\]
Thus, Eq.\eqref{equ:22} holds, yielding that Eq.\eqref{equ:lemma2} also holds.  \qed


\begin{theorem}[Dynamic Regret Bound with LLM Augmentation]
\label{theo:main}
Assume that the conditions of Lemmas~\ref{lem:expect_real} and~\ref{lemma:LLM} hold. Then, with probability at least \(1-2/T\), the dynamic regret over \(T\) rounds is bounded by
\begin{equation}
    \text{D-Regret}_T \leq LH\,\mathcal{P}_T + 2T\,\tilde{\beta}_H \sqrt{\frac{2d}{H}\log\Bigl(1+\frac{L^2H}{\lambda d}\Bigr)},
\end{equation}
where
\begin{itemize}
    \item \(\mathcal{P}_T = \sum_{a \in \mathcal{A}} \sum_{p=t_0}^{T-1} \|\bm x_a(p)-\bm x_a(p+1)\|_2\) denotes the total variation of the arm contexts,
    \item \(\tilde{\beta}_H = \sqrt{\lambda S} + R\sqrt{2\log\Bigl(T\lceil T/H\rceil\Bigr) + d\log\Bigl(1+\frac{HL^2}{\lambda d}\Bigr)}\).
\end{itemize}
Ignoring logarithmic factors, we have
\[
    \text{D-Regret}_T = \tilde{O}\Bigl(H\,\mathcal{P}_T + \frac{dT}{\sqrt{H}}\Bigr).
\]
In particular, by setting 
\[
    H = H^* = \lfloor (dT/\mathcal{P}_T)^{2/3} \rfloor,
\]
one obtains the near-optimal bound
\[
    \text{D-Regret}_T = \tilde{O}\Bigl(d^{2/3}\mathcal{P}_T^{1/3}T^{2/3}\Bigr).
\]
\end{theorem}

\noindent \paragraph{Proof of Theorem~\ref{theo:main}.} 

Due to Lemma~\ref{lem:expect_real} and the fact that $\bm x_{a^*}(t), \bm x_a(t)$ where $a \in \mathcal{A}$, each of the following holds with probability at least $1 - \delta$,
\begin{equation}
\begin{aligned}
    \langle \bm x_{a^*}(t), \bm P_t \rangle &\leq \langle \hat{\bm x}_{a^*}(t), \bm P_t \rangle 
    + L \sum_{p=t_0}^{t-1} \|\bm x_{a^*}(p) - \bm x_{a^*}(p+1)\|_2 
    + \beta_t \|\bm P_t^*\|_{\bm \psi_{a^*,t-1}^{-1}}, \\
    \langle \bm x_a(t), \bm P_t \rangle &\geq \langle \hat{\bm x}_a(t), \bm P_t \rangle 
    - L \sum_{p=t_0}^{t-1} \|\bm x_a(p) - \bm x_a(p+1)\|_2 
    - \beta_t \|\bm P_t\|_{\bm \psi_{a,t-1}^{-1}}.
\end{aligned}
\end{equation}

By the union bound, the following holds with probability at least $1-2\sigma$,



\begin{align}
\langle \bm x_{a^*}(t), \bm P_t \rangle - \langle \bm x_{a}(t), \bm P_t \rangle 
&\leq \langle \hat{\bm x}_{a^*}(t), \bm P_t \rangle - \langle \hat{\bm x}_{a}(t), \bm P_t \rangle 
+ L \sum_{p=t_0}^{t-1} \|\bm x_a(p) - \bm x_a(p+1)\|_2  
\\&+ L \sum_{p=t_0}^{t-1} \|\bm x_{a^*}(p) - \bm x_{a^*}(p+1)\|_2  + \beta_t (\|\bm P_t\|_{\bm \psi_{a^*,t-1}^{-1}} + \|\bm P_t\|_{\bm \psi_{a,t-1}^{-1}}) 
\\&\leq L \sum_{p=t_0}^{t-1} \|\bm x_a(p) - \bm x_a(p+1)\|_2 + L \sum_{p=t_0}^{t-1} \|\bm x_{a^*}(p) - \bm x_{a^*}(p+1)\|_2 
+ 2\beta_t \|\bm P_t\|_{\bm \psi_{a, t-1}^{-1}}\label{line:-3}
\\&\leq L \sum_{a \in \mathcal{A}}\sum_{p=t_0}^{t-1} \|\bm x_{a}(p) - \bm x_{a}(p+1)\|_2 +  2\beta_t \|\bm P_t\|_{\bm \psi_{a, t-1}^{-1}}
\\&\leq L \sum_{a \in \mathcal{A}}\sum_{p=t_0}^{t-1} \|\bm x_{a}(p) - \bm x_{a}(p+1)\|_2 +  2\beta_t \|\bm P_t\|_{\bm \Gamma_{a, t-1}^{-1}}\label{line:-1}
\end{align}
where the inequality \eqref{line:-3} comes from the following implication of the arm selection criteria,
\begin{equation}
\langle \hat{\bm x}_{a^*}(t), \bm P_t \rangle + \beta_t \|\bm P_t\|_{\bm \psi_{a^*,t-1}^{-1}} 
\leq \langle \hat{\bm x}_{a}(t), \bm P_t \rangle + \beta_t \|\bm P_t\|_{\bm \psi_{a,t-1}^{-1}}.
\end{equation}

And the inequality \eqref{line:-1} comes from the following conclusion:  the variance term of our algorithm is not larger than that without the LLM augumentation.  Formally, for any $t \in \mathcal{T}, a \in \mathcal{A}$:

\begin{equation}
\left[ \bm P_t^\top (\bm \psi_{a,t-1})^{-1} \bm P_t\right]^{\frac{1}{2}} 
\leq 
\left[ \bm P_t^\top (\bm \Gamma)^{-1} \bm P_t\right]^{\frac{1}{2}}.
\end{equation}
where  $\bm \psi_{a,t-1} = \lambda I_d + \bm \phi_{a,t-1} + \bm \phi_{a, \text{LLM}} $, and   $\lambda I_d + \bm{\phi}_{a,t-1}$ is denoted as $\bm \Gamma$.

Hence, based on the inequallity \eqref{line:-1}, dynamic regret within epoch $\mathcal{E}$ is bounded as follows,
\begin{equation}
\begin{aligned}
\text{D-Regret}(\mathcal{E}) &\leq \sum_{t \in \mathcal{E}} L \sum_{a \in \mathcal{A}}\sum_{p=t_0}^{t-1} \|\bm x_{a}(p) - \bm x_{a}(p+1)\|_2 +  2\beta_t \|\bm P_t\|_{\bm \Gamma_{a, t-1}^{-1}}\\
&\leq LH\mathcal{P}(\mathcal{E}) + 2\beta_H \sqrt{2dH \log\left( 1 + \frac{L^2H}{\lambda d} \right)},
\end{aligned}
\label{equ:regret}
\end{equation}
where $\mathcal{P}(\mathcal{E}) = \sum_{a \in \mathcal{A}}\sum_{p=t_0}^{t-1} \|\bm x_{a}(p) - \bm x_{a}(p+1)\|_2 $. The last inequality holds due to the standard elliptical potential.

By taking the union bound over the dynamic regret of all $\lceil T/H \rceil$ epochs, we know that the following holds with probability at least $1 - 2/T$:
\begin{equation}
\begin{aligned}
    \text{D-Regret}_T &= \sum_{s=1}^{\lceil T/H \rceil} \text{D-Regret}(\mathcal{E}_s)\\
    &\leq LH \mathcal{P}_T + 2T \tilde{\beta}_H \sqrt{\frac{2d}{H} \log\left(1 + \frac{L^2 H}{\lambda d}\right)},
\end{aligned}
\end{equation}
where $\tilde{\beta}_H = \sqrt{\lambda S} + R \sqrt{2 \log\left(T \lceil T/H \rceil\right) + d \log\left(1 + \frac{HL^2}{\lambda d}\right)}$. 

Ignoring logarithmic factors, we finally obtain:
\begin{equation}
    \text{D-Regret}_T \leq \tilde{O}(H \mathcal{P}_T + dT / \sqrt{H}).
\end{equation}

By setting $H = H^* = \lfloor(dT / \mathcal{P}_T)^{2/3} \rfloor$, we achieve a:
\begin{equation}
    \tilde{O}(d^{2/3} \mathcal{P}_T^{1/3} T^{2/3}),
\end{equation}
which is a near-optimal dynamic regret.  \qed

Note that in a non-stationary environment, some studies~\cite{cheung2019learning} provide the optimal dynamic regret lower boun: 
$\Omega\big(d^{2/3} P_T^{1/3}T^{2/3}\big)$. 
Therefore, following prior work~\cite{zhao2020simple}, 
we refer to $\tilde{O}(d^{2/3}\mathcal{P}_T^{1/3} T^{2/3})$ as the near-optimal dynamic regret.

\subsection{Empirical Validation}
Inspired by \cite{zhao2020simple}, we already know that in the setting of nonstationary linear bandits, a simple restarted strategy can achieve near-optimal dynamic regret.  A similar modeling approach applies to our HyperBandit+, which generates the parameters of the bandits bansed on the current time, instead of resetting them to zero when detecting enveironment change as is done in the restarted strategy. Under the assumption that the parameters are reset to zero at the beginning of each epoch, we obtain the regret bound as shown in Eq.~\eqref{equ:regret}. Furthermore, as shown 
 in Figure.~\ref{fig:regret}, Our experiments demonstrate that HyperBandit+, which generates model parameters at the start of each epoch, achieves a better regret bound compared to resetting the parameters to zero (i.e., RestartUCB~\cite{zhao2020simple}).

\subsection{Discussion on LLM Start}
The success of LLM Start can be attributed to enabling the algorithm to perform effective exploration on synthetic data, which facilitates better exploitation on real data.  

\begin{enumerate}[(1)]
    \item \textbf{Exploration:} LLM Start leverages the world knowledge of LLMs to complement sparse user interactions.  By exploring on synthetic data, the algorithm gradually reduces uncertainty across each arm, allowing the model to avoid starting from scratch with purely random exploration when applied to real data.  Instead, it performs more targeted exploration in promising regions based on the knowledge acquired from pretraining, thus achieving faster convergence.
    \item \textbf{Exploitation:} After training on synthetic data, the algorithm’s initial parameters already encode a preliminary understanding of user preferences within the synthetic data.  Consequently, it can make more informed decisions from the beginning, leveraging this prior knowledge to recommend items likely to yield high rewards and avoiding poor user experiences caused by early-stage random recommendations.
\end{enumerate}

\begin{figure}
\centering
\begin{subfigure}{0.30\textwidth}
    \centering
    \includegraphics[width=\linewidth]{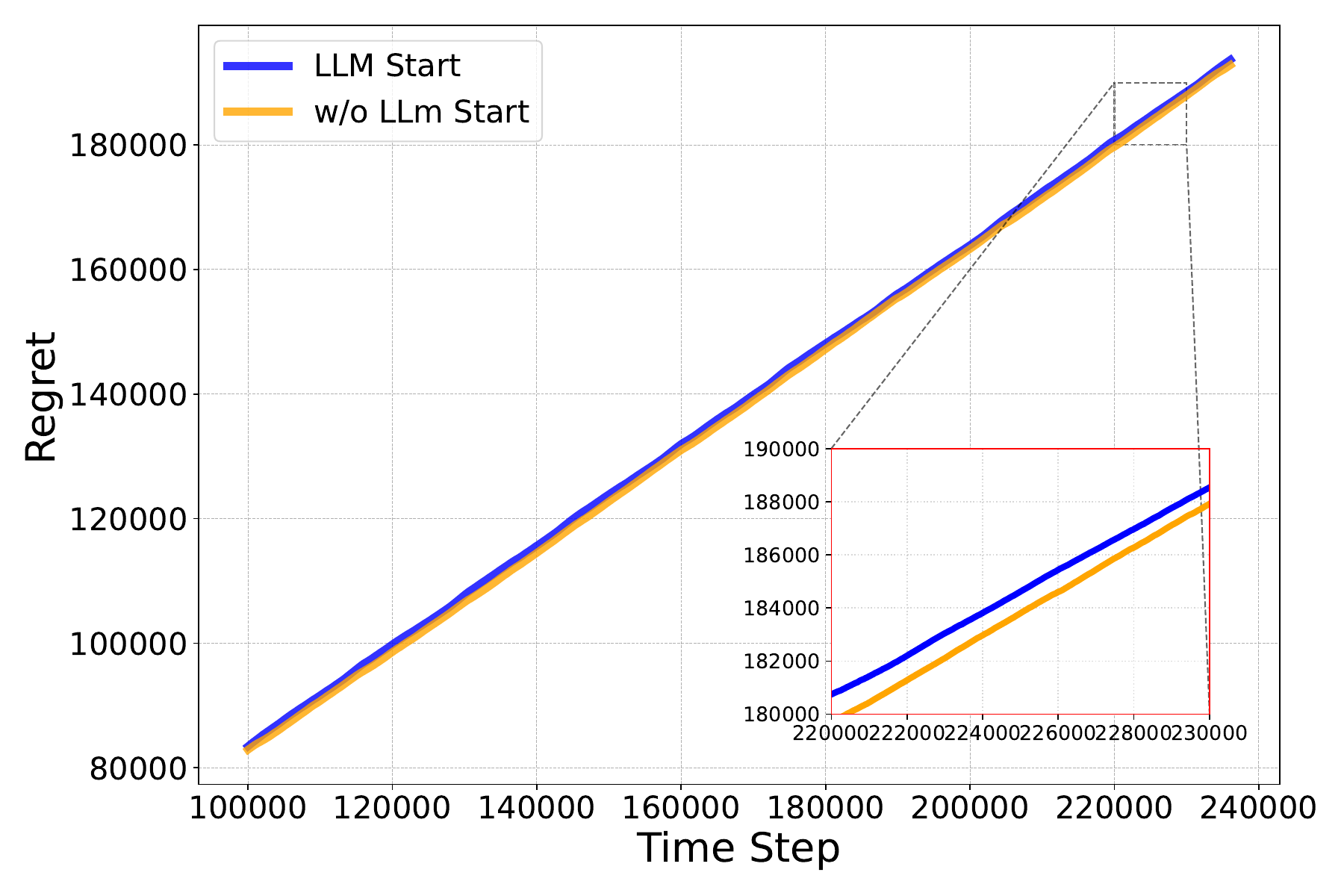}
    \caption{The regret on Kuai}
    \label{fig:kuai:regret}
\end{subfigure}
\hfill
\begin{subfigure}{0.30\textwidth}
    \centering
    \includegraphics[width=\linewidth]{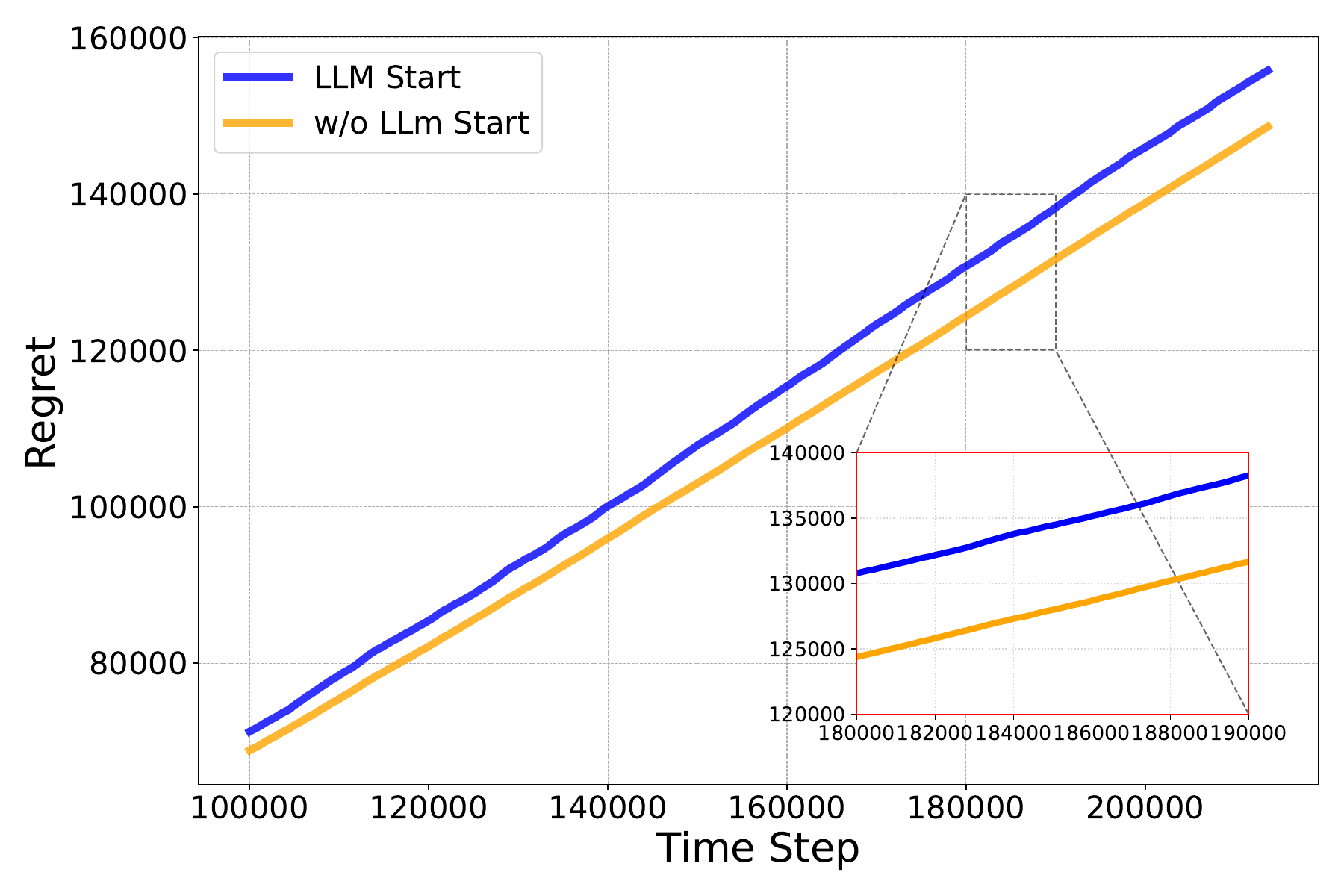}
    \caption{The regret on NYC}
    \label{fig:NYC:regret}
\end{subfigure}
\hfill
\begin{subfigure}{0.30\textwidth}
    \centering
    \includegraphics[width=\linewidth]{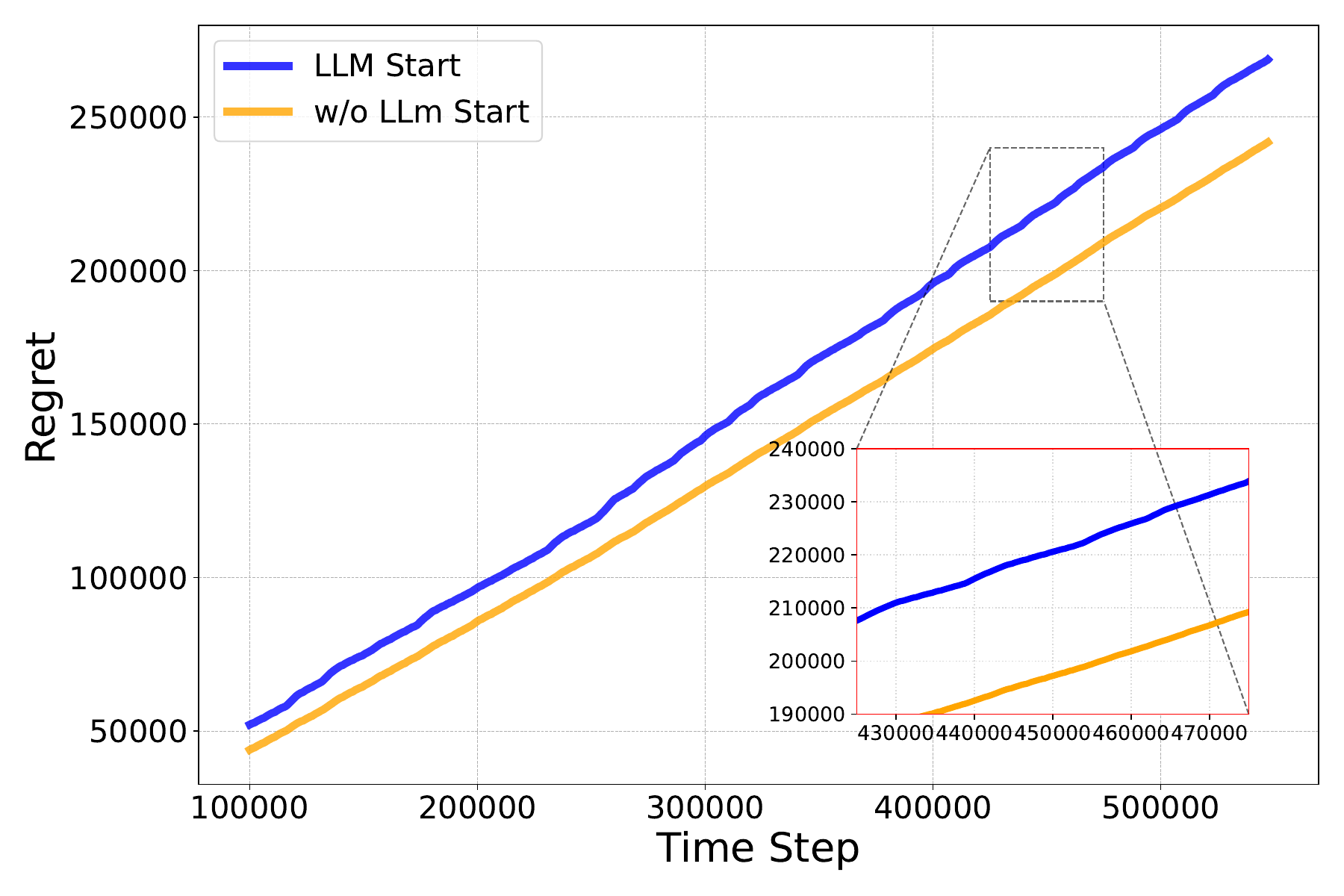}
    \caption{The regret on TKY}
    \label{fig:TKY:regret}
\end{subfigure}
\caption{The result of regret difference between restart and w/o restart settings on three datasets. Where the w/o Start algorithm refers to the RestartUCB algorithm~\cite{zhao2020simple} without the LLM Start component}.
\label{fig:regret}
\end{figure}






\section{Experiments}\label{sec:Experiments}
We conducted experiments to evaluate the performance of HyperBandit+ on datasets for short video recommendation and point-of-interest (POI) recommendation.

\subsection{Experimental Settings}

\subsubsection{Baselines.} HyperBandit+ was compared with several algorithms that constructed in stationary, piecewise-stationary and periodic environment, including:







\textbf{LinUCB \cite{li2010contextual}} is a classical contextual bandit algorithm that addresses the problem of personalized recommendation.

\textbf{HybridLinUCB \cite{li2010contextual}} is a variant algorithm of LinUCB that takes into account both shared and non-shared interests among users. 

\textbf{DLinUCB \cite{wu2018learning}} is built upon a piecewise stationary environment, where each user group corresponds to a slave model. Whether to discard a slave model is based on the detection of ``badness''. Through continuous interaction with the environment, this approach enables adaptive detection of change points and updates the selection strategy.

\textbf{ADTS \cite{hariri2015adapting}} is a bandit algorithm based on Thompson sampling, which tends to discard parameters before the last change point and conducts the recommendation just based on the interactions in the current context. 


\textbf{FactorUCB \cite{wang2017factorization}} leverages observed contextual features and user interdependencies to improve the convergence rate and help conquer cold-start in recommendation.

\textbf{HyperBandit \cite{shen2023hyperbandit}} is an online learning algorithm that dynamically models user preference shifts in periodic non-stationary streaming recommendations via a hypernetwork-enhanced bandit framework.

It is worth noting that our work focuses on enhancing bandit algorithms for streaming recommendation, rather than proposing a general streaming recommender system.
Hence, we follow prior studies~\cite{xu2020contextual}, which also benchmark primarily against bandit-based baselines. Adding non-bandit streaming models would diverge from the scope and formulation of our problem, as those methods typically rely on full feedback or supervised settings, whereas we operate under the partial-feedback (bandit) constraint.  

\subsubsection{Hyperparameter Settings.}
We implemented the hypernetwork $h$ in  Eq.~\eqref{eq:hn:formulation} using a MLP. The MLP consists of 1 input layer, 8 hidden layers, and 1 output layer. The network comprises layers with the following node counts: 4, 256, 512, 1024, 1024, 1024, 1024, 512, 256, and $25 \times \tau \times 2$, where $\tau$ denotes the rank used in low-rank factorization. ReLU activation functions are employed after each hidden layer. The hypernetwork was trained at intervals of 2000 time steps (i.e., $T_n = 2000, n\in [N]$) for KuaiRec and NYC, and every 5000 steps ($T_n = 5000$) for TKY. It is worth noting that the training of the hypernetwork and the ridge regression updates of the arms in the bandit are performed alternately. 
Specifically, ridge regression is updated after each interaction, while the hypernetwork is trained after every $T_n$ interactions. Early stopping was utilized during training to mitigate overfitting. The experiments employed the LLM model Llama-3.1-8B-Instruct\footnote{https://huggingface.co/meta-llama/Meta-Llama-3-8B-Instruct}. Regarding bandit policy parameters, the exploration parameter $\alpha$ and regularization parameter $\lambda$ were uniformly set to 0.1 across all algorithms. For all algorithms, the candidate item set $\mathcal{A}_t$ size at each time step $t \in [T]$ was fixed at 25. Context feature dimensions for users and items were both set to 25. Specifically, for FactorUCB, HyperBandit, and HyperBandit+, latent and observed item feature dimensions were defined as 10 and 15, respectively. 

\subsubsection{Evaluation Protocol.}
The accumulated reward was utilized to assess the recommendation accuracy of algorithms, which was computed as the sum of the observed reward from the beginning to the current step. The normalized accumulated reward refers to the accumulated reward normalized by the corresponding logged random strategy.


\subsection{Experiments on Short Video Recommendation}

We employed KuaiRec~\cite{gao2022kuairec} for evaluation, that is a real-world dataset collected from the recommendation logs of the video-sharing mobile app Kwai\footnote{https://github.com/chongminggao/KuaiRec}.
The dense interaction matrix we used contains 1411 users, 3327 items and 529 video categories (i.e., tags). Each interactive data includes user id, video id, play duration, video duration, time, date, timestamp, and watch ratio, etc. 
Following the settings in~\cite{wan2021contextual}, we used video categories (tags) as actions to reduce the action space. In this experiment, we treated watch ratio higher than 2.0 as positive feedback. If the action (i.e., tag) got positive feedback in other time periods while not in current period, we assumed that the current user would give a negative feedback to it.


 

To fit the data into the contextual bandit setting, we pre-processed it first. To align with the handling approach of POI recommendation, we initially obtain embeddings based on the textual description of items, which serves as the context feature vectors. Following the method described in Section~\ref{sec:method}, we first utilize a large language model (LLM) to enrich the semantic representation of actions (i.e., categories) and infer user profiles based on users' interaction history during the first week. Subsequently, we employ the \textit{text-embedding-ada-002}\footnote{https://platform.openai.com/docs/guides/embeddings} model to obtain contextual feature vectors for both candidate items and users. Principal Component Analysis (PCA) is then applied to reduce the dimensionality of the candidates' context feature vectors, and the resulting PCA projection is similarly used to process the user context feature vectors. We retained the first 25 principal components and applied the same procedure to the context feature vectors of items (i.e., $d_a = d_u =25$ ). To obtain the final user context feature vectors suitable for the dot-product strategy in bandit computation, we extend the previous approach by incorporating the average of the context feature vectors of the categories the user interacted with during the first week. This aggregated value is then used as the final user context feature vector. For a particular time step, the video tag id having positive feedback was picked and the remaining 24 were randomly sampled from the tags which would get negative feedback in the current time step. Additionally, to construct explicit periodic data, we extracted one week's worth of data from August 10th, 2020, to August 16th, 2020. Within each time period of that week, we randomly sampled multiple time steps to reconstruct the dataset while preserving diversity. This sampling process was repeated 16 times to generate data for 16 weeks. 

The results are shown in Fig.~\ref{fig:reward} and Table~\ref{tab:comparisons of NAR and RT}, we can clearly notice that HyperBandit+ outperformed all the other baselines on KuaiRec in terms of rewards, especially in the initial stage. As environment is periodic, both DLinUCB and ADTS were worse than others since these two algorithms were designed for the piece-wise stationary environment (i.e., they need to abandon the knowledge acquired during past periods). As more observations recurrent, LinUCB quickly caught up, because it is better to regard periodic environment as a stationary environment rather than piecewise environment. Besides, FactorUCB leveraged observed contextual features and dependencies among users to improve the algorithm's convergence rate, leading to good performance at the beginning. HyperBandit, as an algorithm designed for periodic environments, also demonstrates strong performance. However, compared to HyperBandit+, it typically experiences a challenging initialization phase at the beginning.

\begin{table}[!t]
    \centering
    \small
    \caption{Comparisons of normalized accumulated reward, running time (sec., mean), and training time (sec., mean) of hypernetwork on KuaiRec, Foursquare (NYC) and Foursquare (TKY). The ``Running Time of BP'' means the average time cost of online recommendation and updating by Bandit Policy at each time step, and the ``Training Time of HN'' means the average time cost for training HyperNetwork at each time step. ``/'' means the corresponding algorithm has no hypernetwork. Statistical significance is evaluated using paired $t$-tests between our method (HyperBadnit+ w/o LowRank) and each baseline over five independent runs. The bolded results indicate significant improvement ($p<0.05$).}
    \vspace{-1ex}
    \label{tab:comparisons of NAR and RT}
    \begin{spacing}{1}
    \resizebox{\columnwidth}{!}{%
    \begin{tabular}{l|c|c|c|c|c|c|c}
        \hline
        \multirow{2}{*}{Algorithm} &
        \multicolumn{3}{c|}{Normalized Accumulated Reward} &
        \multicolumn{2}{c|}{Running Time of BP} &
        \multicolumn{2}{c}{Training Time of HN} \\
        \cline{2-8}
        & KuaiRec & NYC & TKY & KuaiRec & Foursquare & KuaiRec & Foursquare \\
        \hline\hline
        LinUCB & $3.02\pm0.03$ & $4.70\pm0.10$ & $10.44\pm0.03$ & $4.55\mathrm{e{-}04}$ & $4.58\mathrm{e{-}04}$ & / & / \\
        HybridLinUCB & $2.53\pm0.02$ & $4.04\pm0.03$ & $10.31\pm0.10$ & $2.97\mathrm{e{-}02}$ & $2.87\mathrm{e{-}02}$ & / & / \\
        DLinUCB & $2.22\pm0.01$ & $4.07\pm0.05$ & $7.96\pm0.02$ & $5.61\mathrm{e{-}04}$ & $5.15\mathrm{e{-}04}$ & / & / \\
        ADTS & $1.46\pm0.01$ & $2.90\pm0.01$ & $6.49\pm0.04$ & $3.78\mathrm{e{-}03}$ & $3.39\mathrm{e{-}03}$ & / & / \\
        FactorUCB & $3.02\pm0.01$ & $4.30\pm0.14$ & $9.70\pm0.24$ & $9.54\mathrm{e{-}02}$ & $9.53\mathrm{e{-}02}$ & / & / \\
        \hline
        HyperBandit & $4.55\pm0.08$ & $7.99\pm0.17$ & $13.93\pm0.08$ & $2.13\mathrm{e{-}03}$ & $2.22\mathrm{e{-}03}$ & $3.26\mathrm{e{-}04}$ & $3.71\mathrm{e{-}04}$ \\
        \textbf{HyperBandit+} ($\tau=1$) & $4.80\pm0.19$ & $6.91\pm0.06$ & $13.27\pm0.12$ & $3.80\mathrm{e{-}03}$ & $5.37\mathrm{e{-}03}$ & $3.17\mathrm{e{-}04}$ & $1.91\mathrm{e{-}04}$ \\
        \textbf{HyperBandit+} ($\tau=3$) & $4.83\pm0.03$ & $7.91\pm0.14$ & $13.92\pm0.08$ & $3.65\mathrm{e{-}03}$ & $4.75\mathrm{e{-}03}$ & $3.25\mathrm{e{-}04}$ & $2.01\mathrm{e{-}04}$ \\
        \textbf{HyperBandit+} ($\tau=5$) & $4.92\pm0.15$ & $8.27\pm0.08$ & $13.99\pm0.05$ & $3.81\mathrm{e{-}03}$ & $5.58\mathrm{e{-}03}$ & $3.72\mathrm{e{-}04}$ & $2.07\mathrm{e{-}04}$ \\
        \textbf{HyperBandit+} w/o Low-Rank & $\bm{5.08\pm0.11}$ & $\bm{8.49\pm0.25}$ & $\bm{14.00\pm0.05}$ & $3.71\mathrm{e{-}03}$ & $4.99\mathrm{e{-}03}$ & $4.11\mathrm{e{-}04}$ & $2.54\mathrm{e{-}04}$ \\
        \hline
    \end{tabular}%
    }
    \end{spacing}
    \vspace{-1ex}
\end{table}

\subsection{Experiments on POI Recommendation}

\begin{table}[t]
\centering
\caption{The statistics of the Foursquare NYC \& TKY.}
\vspace{-1ex}
    \label{tab:foursquare data info}
    \begin{tabular}{ccccc}
    \toprule
    Dataset & \#Users & \#POIs & \#POI Categories & \#Check-ins  \\
    \hline
    NYC & 1,083 & 38,333 & 400 & 227,428    \\
    TKY & 2,293 & 61,858 & 385 & 573,703    \\
    \bottomrule
    \end{tabular}
\end{table}

Foursquare NYC \& TKY dataset~\cite{yang2014modeling} includes long-term (about 10 months) check-in data in New York city (NYC) and Tokyo (TKY) collected from Foursquare\footnote{https://foursquare.com/} from 12 April 2012 to 16 February 2013. 
Table~\ref{tab:foursquare data info} shows the statistics of two check-in datasets: NYC and TKY. Each dataset includes user id, venue id, venue category id, venue category name, latitude, longitude, and timestamp, etc.


Similarly, we used POI categories as actions. The ground-truth categories of the check-ins were considered positive samples of the current step while the rest categories were considered negative. Initially, we adopt the same processing method as mentioned in last section(i.e., the processing method on KuaiRec dataset) to obtain both the user context features and the context feature vectors of candidate items.  For the candidate action set at each time step, we selected the ground-truth check-in tag and randomly extracted 24 negative categories of the current step. We used all the data from 12 April 2012 to 16 February 2013 (except data from the first week) to construct a data streaming. It is important to note that HyperBandit+ is designed for users with sparse interactions. Therefore, unlike the setting in~\cite{shen2023hyperbandit}, where users who did not appear in the first week were completely excluded, we instead leverage each user's first-week interaction data for enrichment and initial embedding. During the testing phase, we filter out each user's first-week interaction data accordingly.

\begin{figure*}[t]
\centering
\begin{subfigure}[b]{0.32\textwidth}
\includegraphics[width=\textwidth]{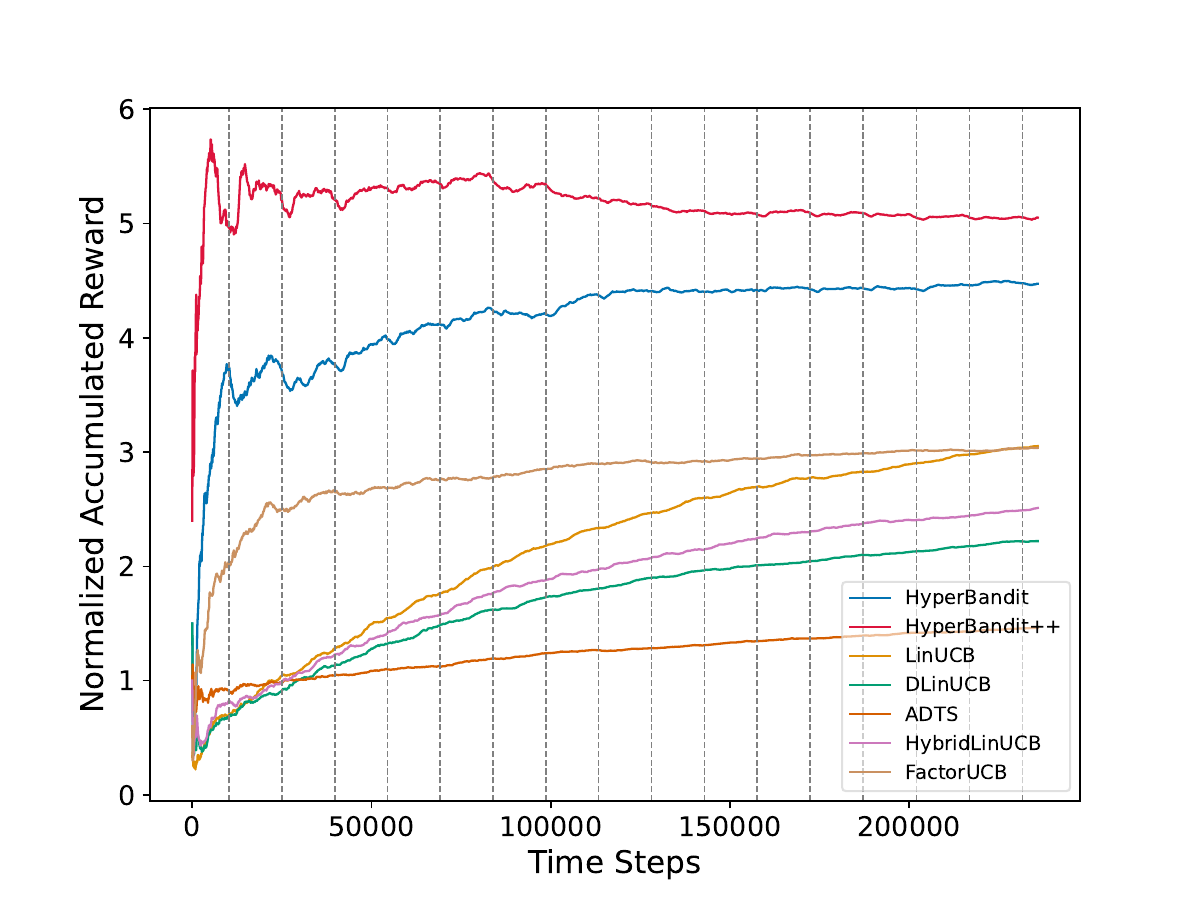}
\caption{ KuaiRec}
\label{subfig:reward:a}
\end{subfigure}
\hfill
\begin{subfigure}[b]{0.32\textwidth}
\includegraphics[width=\textwidth]{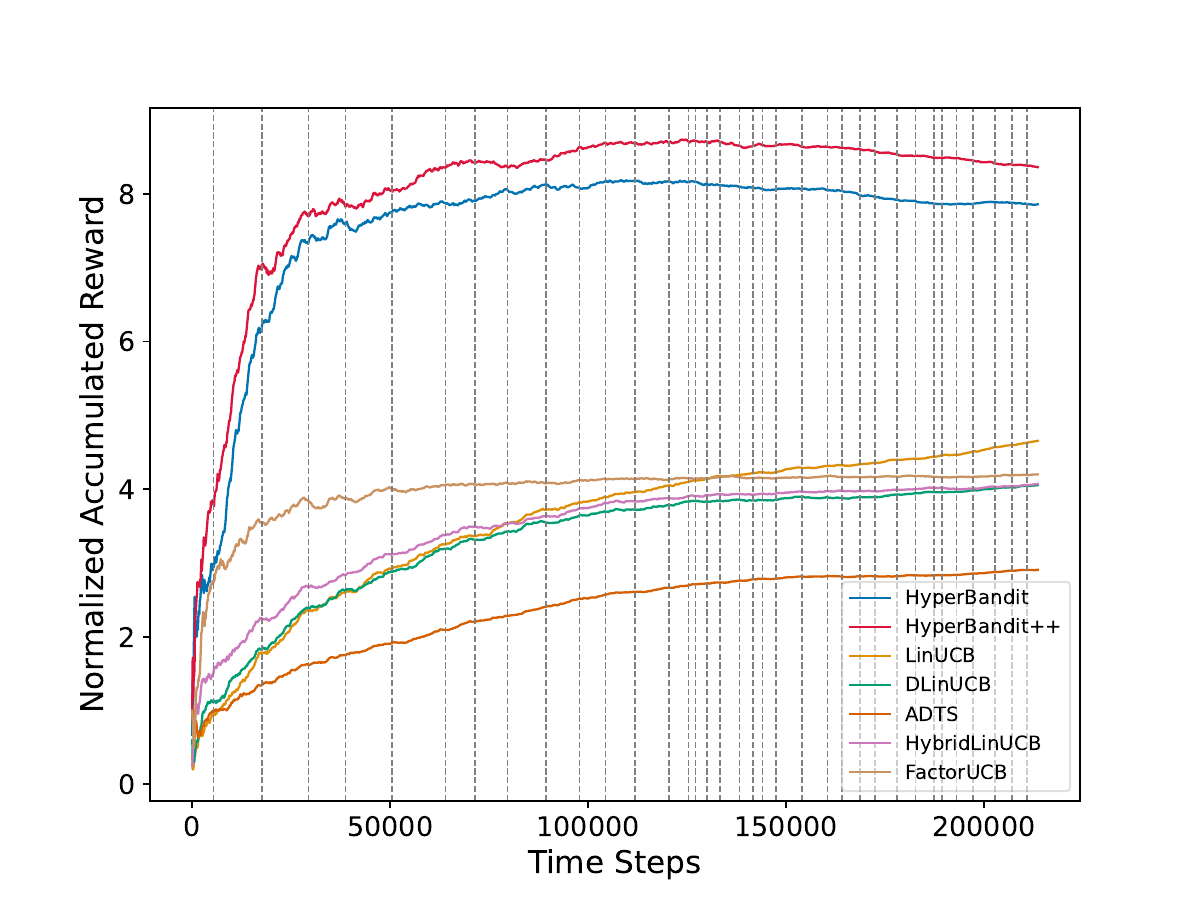}
\caption{ NYC}
\label{subfig:reward:b}
\end{subfigure}
\hfill
\begin{subfigure}[b]{0.32\textwidth}
\includegraphics[width=\textwidth]{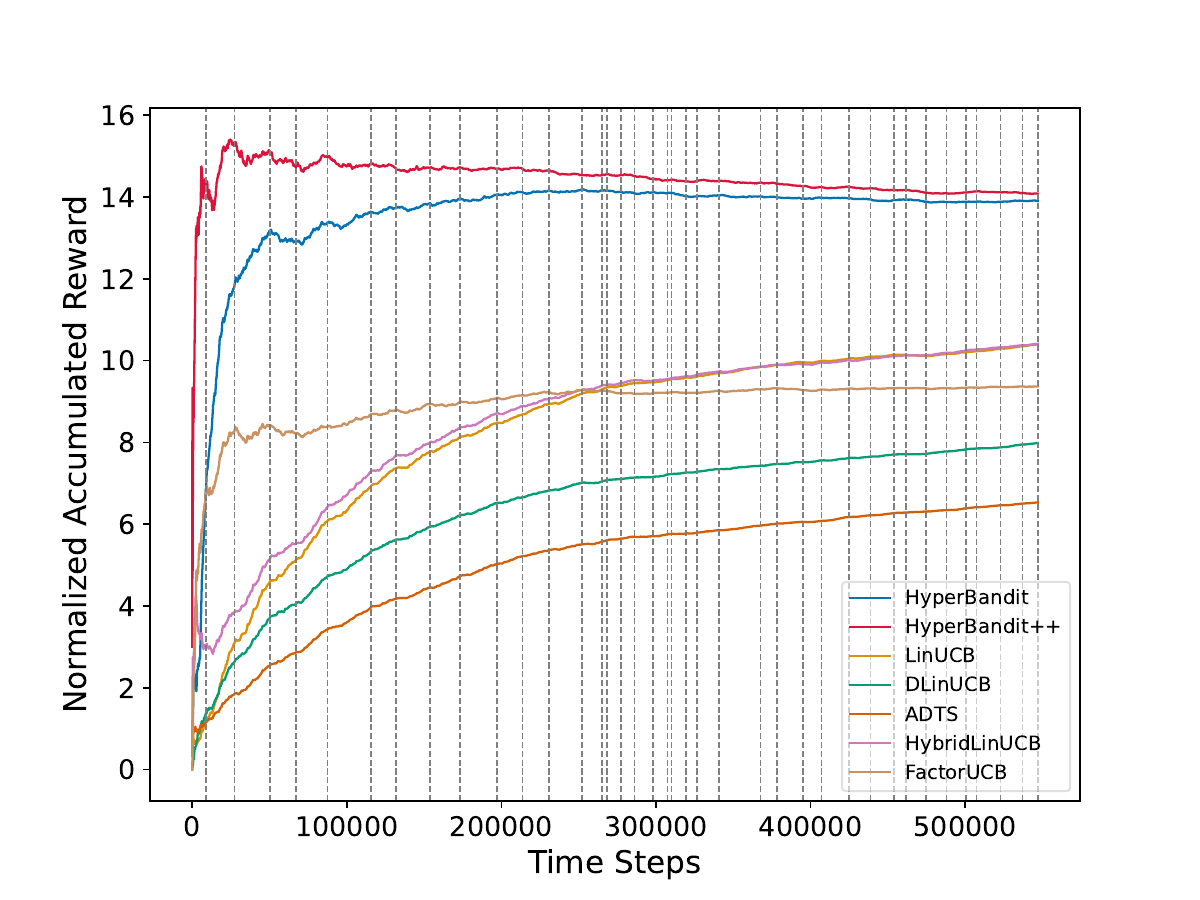}
\caption{ TKY}
\label{subfig:reward:c}
\end{subfigure}
\caption {Normalized accumulated reward of baselines, and the proposed HyperBandit+ on three datasets, KuaiRec \& NYC \& TKY. Note that The grey dashed lines represent the boundaries between weekdays and weekends. The $x$-axis represents the interaction data arranged in chronological order, and the $y$-axis represents the normalized accumulated reward. }

\label{fig:reward}
\end{figure*}

As illustrated in Fig.~\ref{subfig:reward:b} and Fig.~\ref{subfig:reward:c}, similar conclusions can be drawn as in KuaiRec. Additionally, we conducted an analysis of time cost for all algorithms and compared the performance of different estimated rank  ($\tau=1$, $\tau=5$ and w/o Low-Rank) of HyperBandit. The corresponding results are presented in Table~\ref{tab:comparisons of NAR and RT}. Notably, HyperBandit consistently outperformed the baselines in terms of normalized accumulated reward, while the running time of BP (bandit policy) remained acceptable. FactorUCB achieved excellent performance among baselines due to leveraging user adjacency relationships. However, that also lead to a significant time cost as FactorUCB required updating all user parameters at each time step. Furthermore, HyperBandit+ ($\tau=1$), HyperBandit+ ($\tau=3$), and HyperBandit+ ($\tau=5$) reduced the training time of HN by 22. 8\%, 20. 9\%, 9. 5\% in KuaiRec and, 24. 8\%, 20. 8\%, 18. 5\% in Foursquare dataset compared to HyperBandit+ (w / without Low rank), which provided strong evidence of the training efficiency with low rank factorization.

\subsection{Ablation Experiments}

In this section, we empirically studied the proposed HyperBandit+ by addressing the following research questions: 

{\raggedright\textbf{RQ1:}} Is HyperBandit+ efficient enough to meet the real-time requirements of online recommendations?

{\raggedright\textbf{RQ2:}} 
How does the estimated rank $\tau$ of low-rank factorization affect HyperBandit+?

{\raggedright\textbf{RQ3:}} Does the LLM Start mechanism truly improve exploration and exploitation in the initial stage?

{\raggedright\textbf{RQ4:}} What is the impact of key embedding methods in HyperBandit+?

{\raggedright\textbf{RQ5:}} Have the three extensions in HyperBandit+ fulfilled their respective roles?

{\raggedright\textbf{RQ6:}} What is the impact of key updating methods
in HyperBandit+ on the recommendation performance?

\subsubsection{\textbf{RQ1: Running Time}} 
In streaming recommendation scenarios, running time is a crucial metric. We report the running time of the bandit policy and the training time of the hypernetwork in Table~\ref{tab:comparisons of NAR and RT}. The results indicate that the time cost of HyperBandit+ is on the order of milliseconds (ms), demonstrating its capability to meet real-time requirements in streaming recommendation settings. Moreover, the training time of the hypernetwork exhibits a decreasing trend as the estimated rank \(\tau\) is progressively reduced from full rank to 5, then to 3, and finally to 1, further confirming its efficiency in low-rank factorization.

\begin{figure*}[t]
\centering
\begin{subfigure}[b]{0.32\textwidth}
\includegraphics[width=\textwidth]{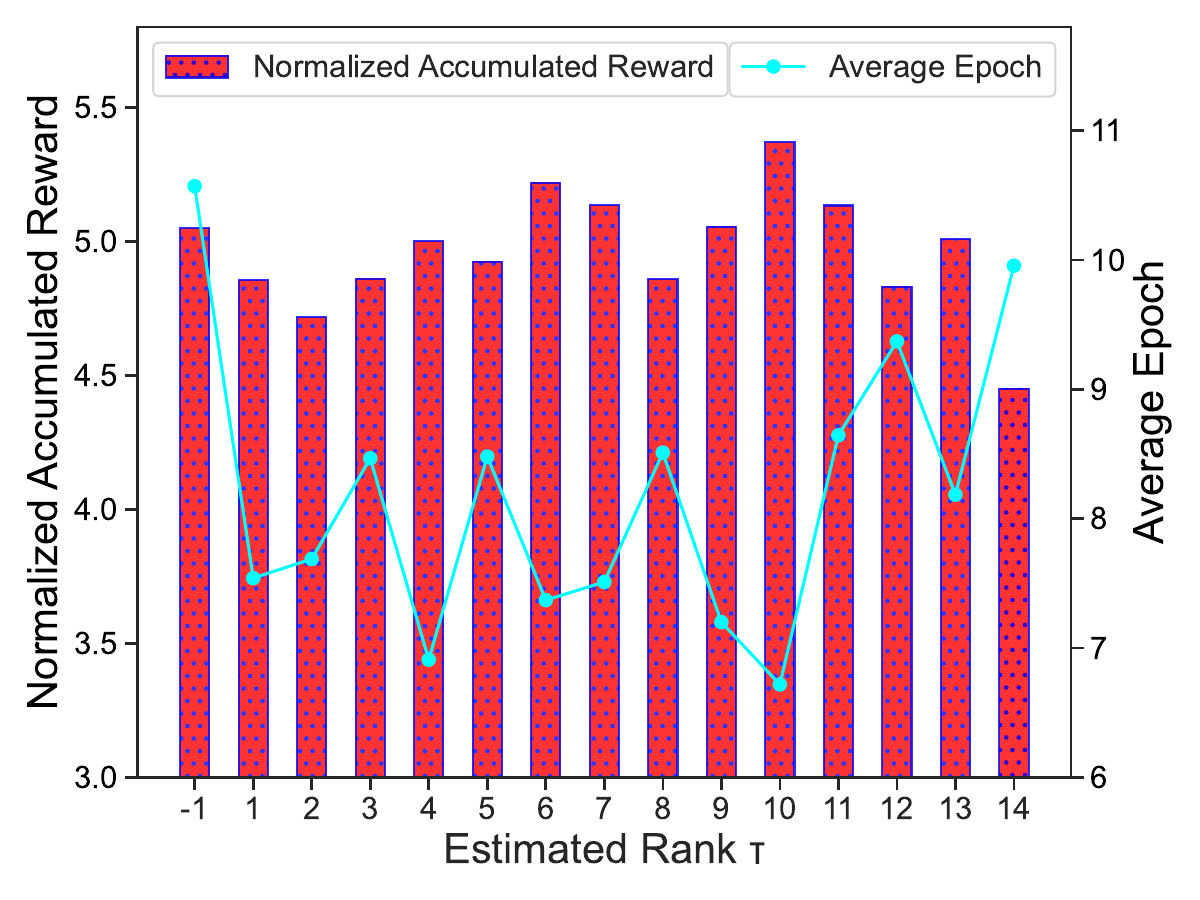}
\caption{KuaiRec}
\end{subfigure}
\hfill
\begin{subfigure}[b]{0.32\textwidth}
\includegraphics[width=\textwidth]{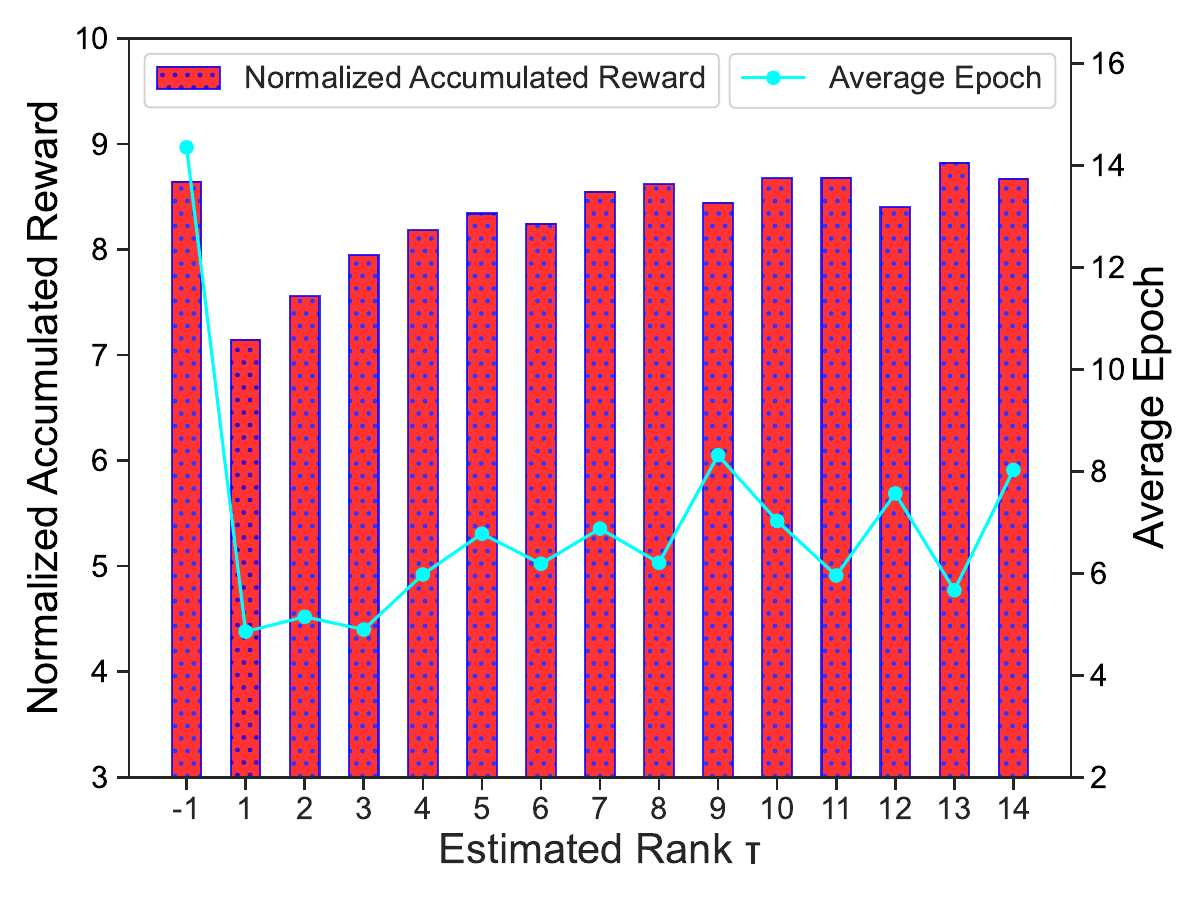}
\caption{NYC}
\end{subfigure}
\hfill
\begin{subfigure}[b]{0.32\textwidth}
\includegraphics[width=\textwidth]{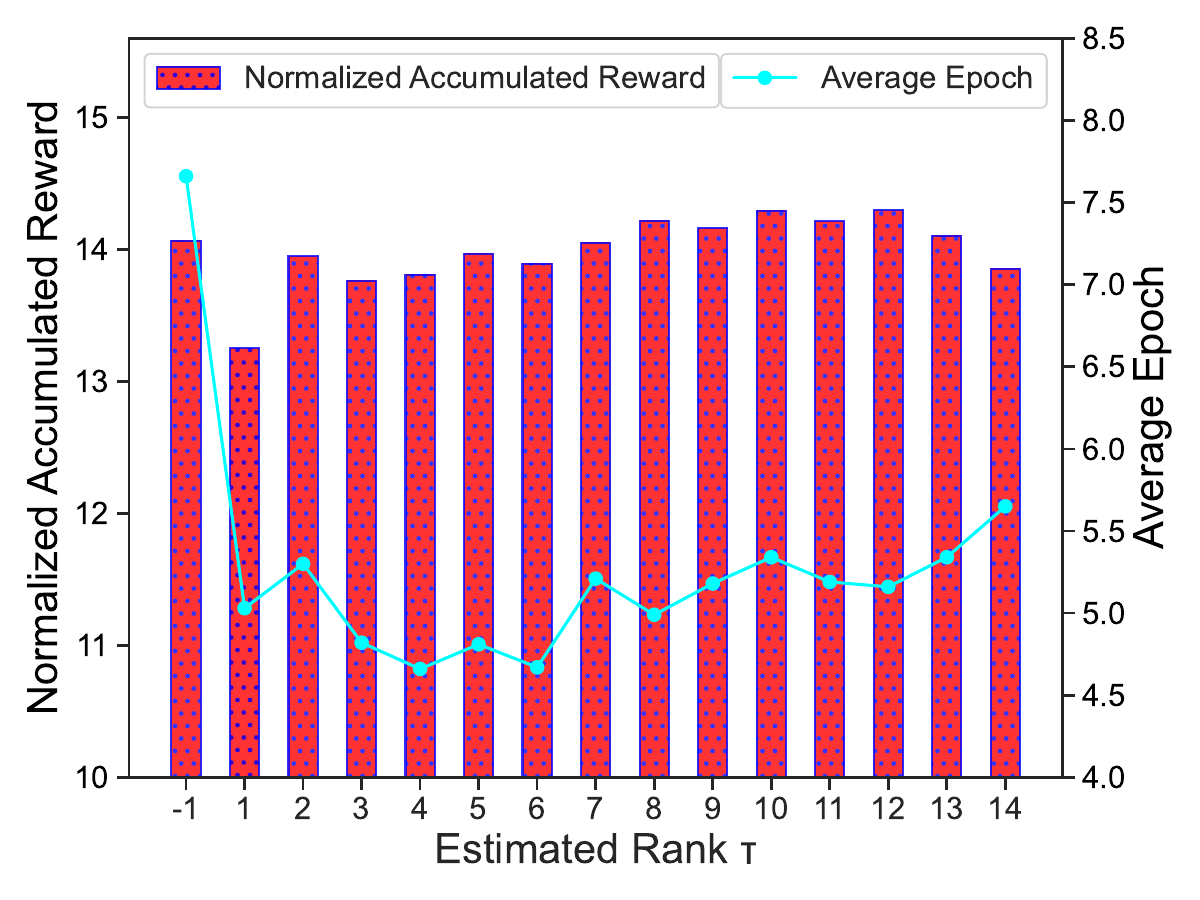}
\caption{TKY}
\end{subfigure}
\vspace{-2ex}
\caption{Performance of low-rank factorization in HyperBandit on different estimated rank across three datasets. Note that the bar chart shows normalized accumulated reward, while the line chart shows average epoch in training process (a larger average epoch indicates a longer training time of hypernetwork). The first data point (with an $x$-coordinate of ``$\mathrm{-1}$'') represents the result obtained without utilizing low-rank factorization.}
\label{fig:ranks performance}
\vspace{-2.2ex}
\end{figure*}

\subsubsection{\textbf{RQ2: Impact of  Estimated Rank $\tau$}}

To investigate the impact of low-rank factorization with varying estimated ranks on HyperBandit+, we conducted experiments across three datasets.The estimated rank $\tau$ was tested with $\mathrm{-1}$ and within a range of $\mathrm{1}$ to $\mathrm{14}$, where $\tau$=$\mathrm{-1}$ implies not utilizing low-rank factorization. Performance was evaluated using two metrics: normalized accumulated reward and training time.

The observations from the
experimental results, presented in Fig.~\ref{fig:ranks performance}, can be summarized as follows: 1). With an increase in the estimated rank $\tau$, our HyperBandit+
demonstrated an overall improvement in normalized accumulated reward on Foursquare datasets, and the normalized
accumulated reward of HyperBandit+ with ranks ranging from 5 to 14 were nearly equivalent to that of HyperBandit+
without low-rank factorization. Particularly on KuaiRec dataset, the normalized accumulated reward of HyperBandit+ with ranks from 1 to 9 shows a gradual increase, even surpassing the algorithm without low-rank factorization. However, for ranks from 11 to 14, the reward declines slightly, suggesting potential overfitting or computational redundancy at higher ranks. These results validate the effectiveness of the low-rank factorization approach, which maintains excellent performance.2). The average epoch, which measures the training time of the hypernetwork, also exhibits an
overall upward trend as the estimated rank $\tau$ increases, although it is significantly smaller than that without low-rank factorization. This observation highlights the benefits of efficient training via low-rank factorization as described in Sec.~\ref{sec:HB:Lowrank:training}.



\subsubsection{\textbf{RQ3: Does the LLM Start Mechanism Truly Improve Exploration and Exploitation in the Initial Stage?}}

To investigate whether synthetic data generated by LLMs effectively improves parameter initialization for Bandit models, we present the normalized accumulated reward over the first 10000 rounds on real data under two settings: with and without the LLM Start mechanism(i.e., LLM Start and w/o LLM Start).
\begin{figure*}[t]
\centering
\begin{subfigure}[b]{0.32\textwidth}
\includegraphics[width=\textwidth]{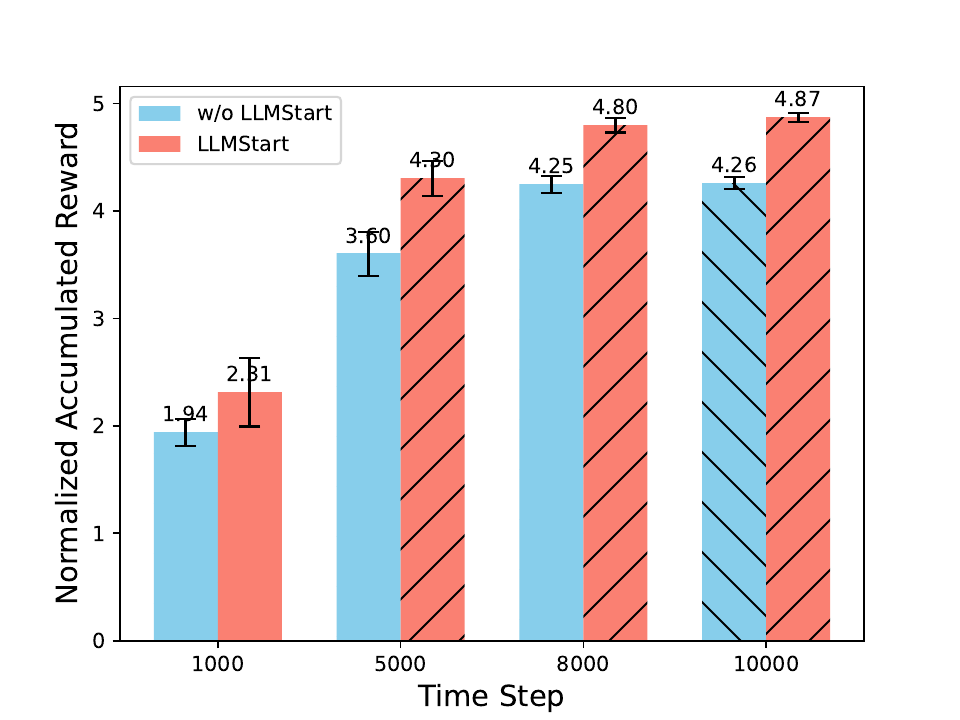}
\caption{KuaiRec}
\end{subfigure}
\hfill
\begin{subfigure}[b]{0.32\textwidth}
\includegraphics[width=\textwidth]{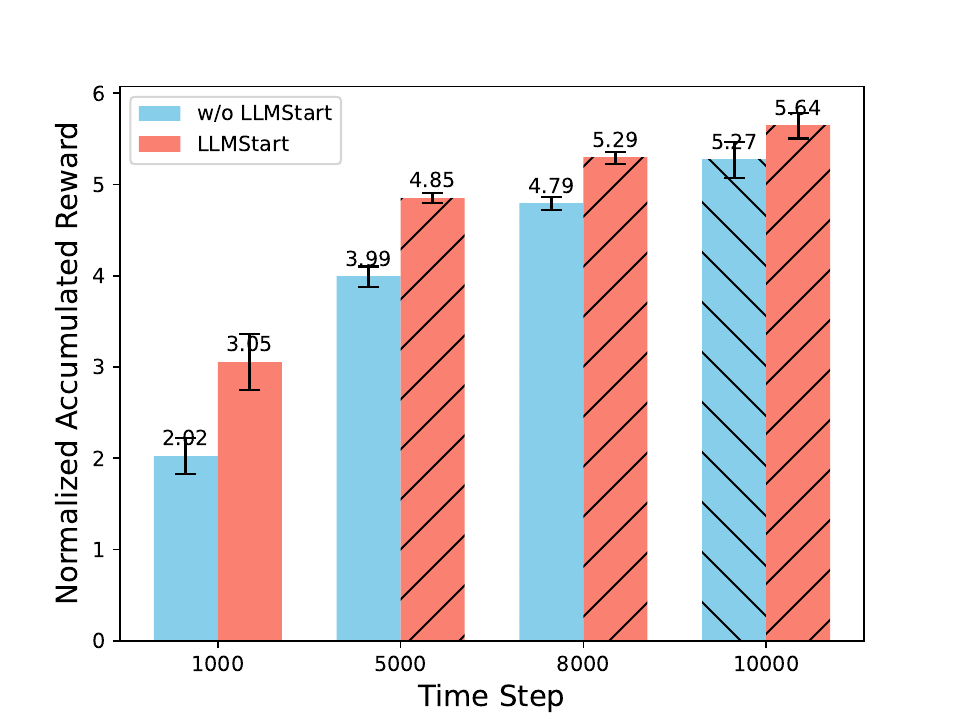}
\caption{NYC}
\end{subfigure}
\hfill
\begin{subfigure}[b]{0.32\textwidth}
\includegraphics[width=\textwidth]{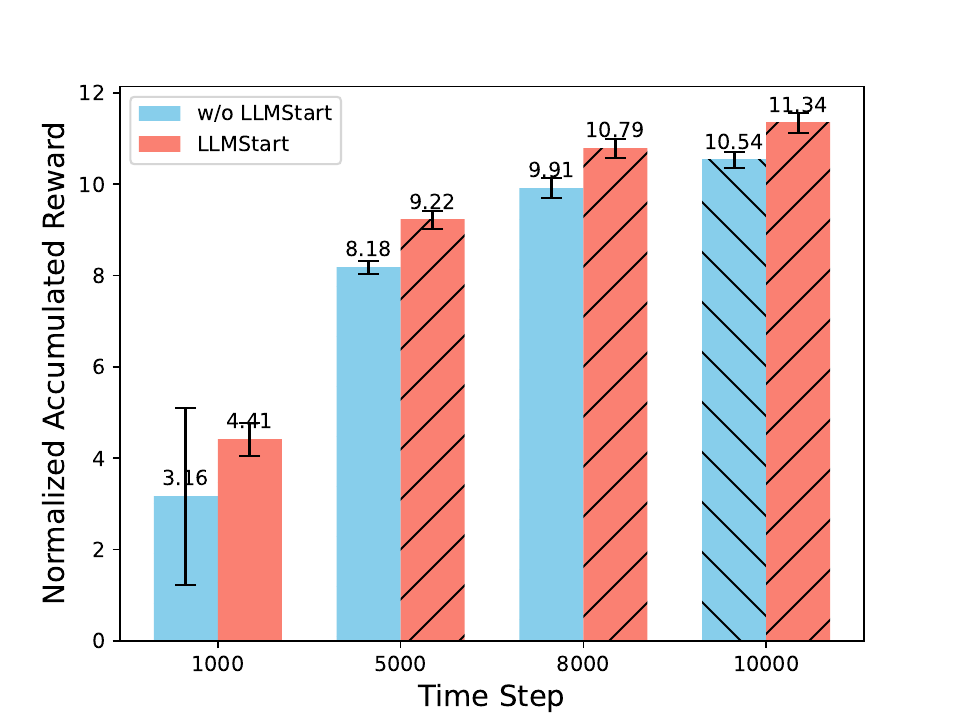}
\caption{TKY}
\end{subfigure}
\vspace{-2ex}
\caption{Performance of LLM Start in HyperBandit+ across three datasets.  ``LLM Start'' refers to HyperBandit+ leveraging synthetic data generated by an LLM to update internal parameters before the online phase on real datasets, while ``w/o LLM Start'' indicates HyperBandit+ with randomly initialized parameters applied directly to real datasets.}
\label{fig:LLMStartinhp+}
\vspace{-2.2ex}
\end{figure*} 
The results are illustrated in Fig.\ref{fig:LLMStartinhp+}, from which we can make the following observation: across all three datasets, HyperBandit+ with the LLM Start mechanism achieves higher rewards in the initial stage compared to the version without LLM Start. This validates that LLM Start effectively mitigates the sparsity of real data, enabling HyperBandit+ to exhibit improved exploration and exploitation capabilities during the initial online phase. Note that LLM Start essentially depends on the quality of the synthetic data, and its performance gain further demonstrates the effectiveness of the synthetic data.

\subsubsection{\textbf{RQ4: Impact of Key Embedding Methods in HyperBandit+.}}
Compared to HyperBandit, HyperBandit+ enhances the model by incorporating multiple embedding techniques, including the use of Euler embedding for time embedding—which better emphasizes periodicity. We conduct ablation experiments comparing the Euler embedding approach with a simple one-hot embedding method.
Additionally, the adoption of LLM-enhanced embeddings for user and item context feature, which leverage richer world knowledge to enrich information and harness powerful embedding capabilities. We conducted extensive ablation experiments to demonstrate the effectiveness of these embeddings.The results are presented in Table~\ref{tab:emb}, from which we formally emphasize two key improvements:  1) Euler embedding enhanced the hypernetwork's ability to capture periodic time patterns.  
2) LLM-enhanced embedding led to an improvement in the normalized accumulated reward.

\begin{table}[t]
    \centering
    \caption{Comparison of normalized accumulated reward (mean~$\pm$~std) across KuaiRec, NYC and TKY datasets under different embedding configurations. ``Euler'' refers to the use of Euler embedding for temporal representation, where ``\xmark/$\checkmark$'' indicate the use of one-hot embedding and Euler embedding, respectively. ``LLM'' denotes the adoption of LLM-enhanced embeddings for contextual feature vectors, where ``\xmark/$\checkmark$'' represent whether LLM-enhanced embeddings are used or not. Statistical significance is evaluated using paired t-tests between our method and each ablation variant over five independent runs. The bolded results indicate significant improvement ($p < 0.05$).}
    \label{tab:emb}
    \begin{subtable}[ht]{\linewidth}
    \centering
    \begin{tabular}{c|c|c|c|c}
        \hline
        \multirow{2}{*}{Euler} & \multirow{2}{*}{LLM} & \multicolumn{3}{c}{Normalized Accumulated Reward}  \\
        \cline{3-5}
         ~&~& KuaiRec &   NYC   & TKY  \\ 
        \hline \hline
         \checkmark       & \checkmark       &$\bm{5.08\pm 0.11}$  & $\bm{8.49\pm 0.25}$  &$14.00\pm 0.05$    \\
         \hline
        \checkmark       & \xmark       &$4.78\pm 0.20$  & $8.18\pm 0.24$  &$13.90\pm 0.14$    \\
        \xmark     & \checkmark       & $4.86\pm 0.20$   & $8.25\pm 0.05$  & $13.99\pm 0.03$   \\ 
        \hline
    \end{tabular}
    \end{subtable}
\end{table}

\begin{figure*}[t]
\centering
\begin{subfigure}[b]{0.32\textwidth}
\includegraphics[width=\textwidth]{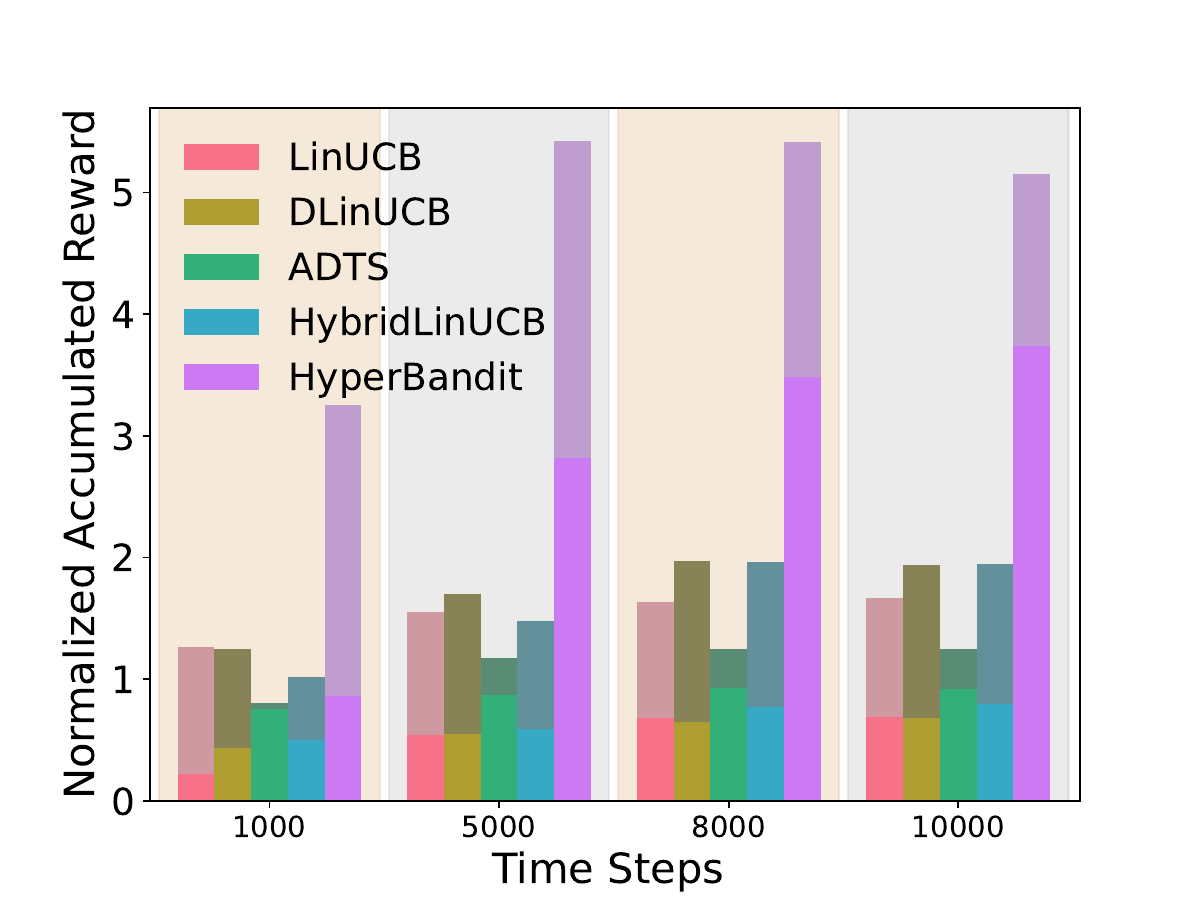}
\caption{KuaiRec}
\end{subfigure}
\hfill
\begin{subfigure}[b]{0.32\textwidth}
\includegraphics[width=\textwidth]{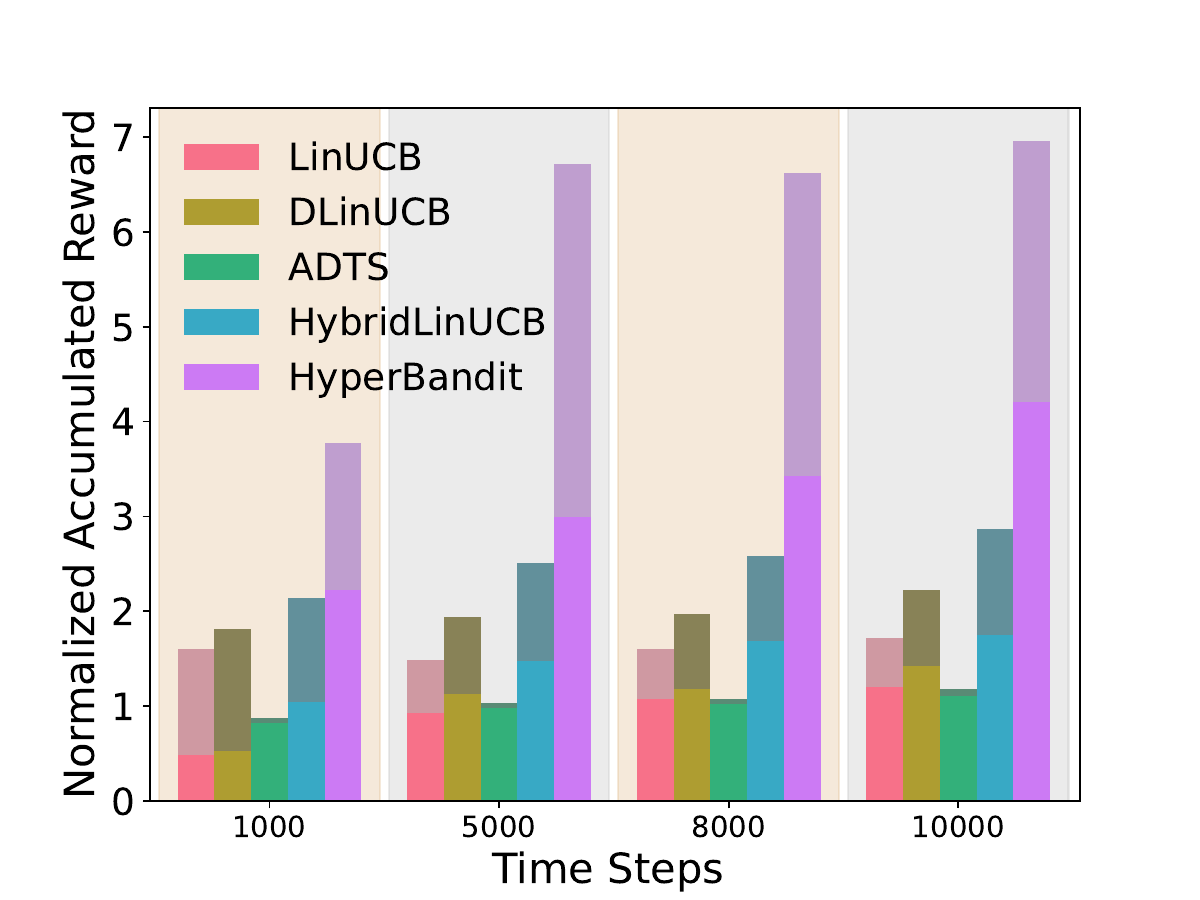}
\caption{NYC}
\end{subfigure}
\hfill
\begin{subfigure}[b]{0.32\textwidth}
\includegraphics[width=\textwidth]{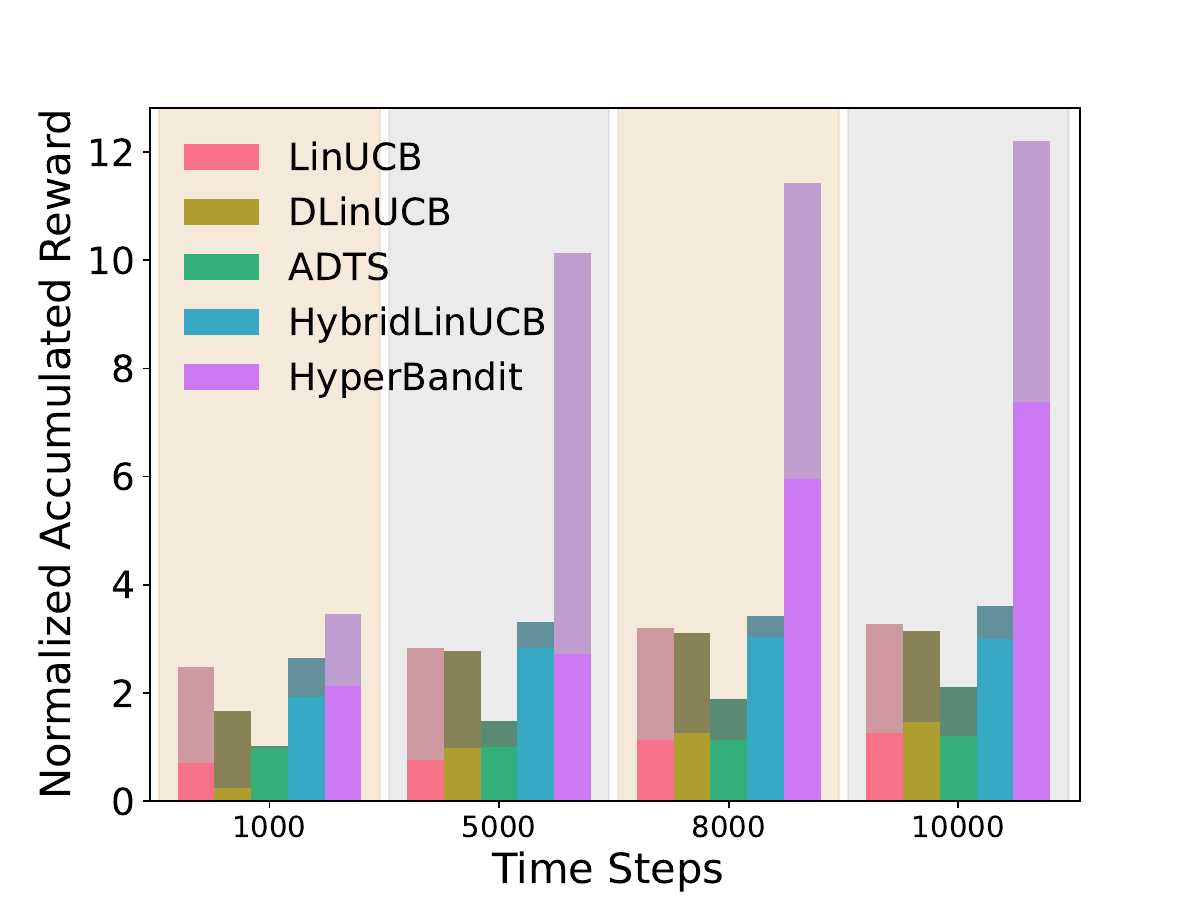}
\caption{TKY}
\end{subfigure}
\vspace{-2ex}
\caption{Normalized accumulated reward comparison of baselines  with and without LLM Start mechanism across three datasets over the first 10000 steps. Each stacked bar comprises two components: the lower solid segment denotes baseline performance without LLM Start initialization, while the upper light-tinted segment quantifies the additive performance gains achieved with LLM Start mechanism.}
\label{fig:LLMStart}
\vspace{-2.2ex}
\end{figure*}
\subsubsection{\textbf{RQ5: How Generalizable is the LLM Start Mechanism?}}
In this section, we investigate the adaptability of the LLM Start mechanism across various baselines. Specifically, we first allow multiple algorithms to perform offline learning on synthetic data generated by an LLM, updating their parameters, before applying these algorithms to online real data. We then observe the performance of these algorithms over the first 10000 steps.The results are presented in Figure~\ref{fig:LLMStart}, leading to the following conclusion: The LLM Start mechanism improves various bandit strategies, demonstrating its generalizability. Additionally, this provides valuable insight: LLM Start significantly benefits the initial phase of reinforcement learning (RL) algorithms, and leveraging LLM-generated synthetic data for RL algorithm initialization is an effective approach. By applying the synthetic data to various baseline algorithms, the consistent performance gains demonstrate the generalization capability of the synthetic data, further validating the effectiveness of the LLM-generated data.

\subsubsection{\textbf{RQ6: Impact of Key Online Updating Methods in HyperBandit+.}}
HyperBandit+ comprises two key components for online updating: (1) updating the latent features of items via ridge regression and (2) updating the parameters of the hypernetwork through gradient descent.  To examine the interaction between these two updating mechanisms, we conducted an ablation experiment with the following settings:  1) Disabling ridge regression updating: The dimensionality of the latent item features was set to zero. 2) Disabling hypernetwork updating: The hypernetwork parameters were frozen at their initial values.  The results, presented in Table~\ref{tab:comparisons of NAR and RR}, lead to the following conclusions:  
1) Independently employing either updating mechanism enhances recommendation performance.  
2) Regardless of whether ridge regression was enabled or disabled, the use of the hypernetwork consistently improved performance, demonstrating the effectiveness of the hypernetwork.

                                  

\begin{table}[t]
    \centering
    \caption{The results of the ablation experiment on key components of HyperBandit+. Note that the ridge regression updating in the bandit policy is denoted by ``RR'', and the hypernetwork updating is referred to as ``HN''. The symbol \checkmark~ signifies the inclusion of a particular update process, while the symbol \xmark~ indicates its exclusion. Note that since the latent features are entirely learned through the ``RR'' mechanism, excluding ``RR'' represents removing the use of latent features. Statistical significance is evaluated using paired t-tests between our method and each ablation variant over five independent runs. The bolded results indicate significant improvement ($p < 0.05$).
} 
    \label{tab:comparisons of NAR and RR}
    \begin{subtable}[ht]{\linewidth}
    \centering
    \begin{tabular}{c|c|c|c|c}
        \hline
        \multirow{2}{*}{RR} & \multirow{2}{*}{HN} & \multicolumn{3}{c}{Normalized Accumulated Reward}  \\
        \cline{3-5}
         ~&~& KuaiRec &   NYC   & TKY  \\ 
        \hline \hline
        \xmark      & \xmark      & $0.86\pm0.14$  & $0.61\pm 0.18$   &  $1.35\pm 0.37$      \\
        \checkmark       & \xmark       &$4.38\pm 0.26$  & $6.57\pm 0.12$  &$13.73\pm 0.13$    \\
        \xmark     & \checkmark       & $3.32\pm 0.08$   & $6.10\pm 0.16$  & $11.04\pm 0.45$   \\ 
        \checkmark    & \checkmark      & $\bm{5.08\pm 0.11 }$ & $\bm{8.49\pm 0.25}$ &  $\bm{14.01\pm 0.05}$\\
        \hline
    \end{tabular}
    \label{tab:bt:bandit:SBUCB}
    \end{subtable}
\end{table}


\subsubsection{\textbf{RQ7: Impact of Hyperparameters.}}

\begin{table}[t]
\caption{The results of the analysis experiments on candidate pool size, time segmentation, and LLM Start scale.}
\label{tab:revision}
\begin{tabular}{cc|cc|cc}
\toprule
\multicolumn{6}{c}{Normalized Accumulated Reward in NYC(NAR)} \\
\midrule
\multicolumn{1}{c}{Candidate Pool Size} &
\multicolumn{1}{c|}{NAR} &
\multicolumn{1}{c}{Time Segmentation} &
\multicolumn{1}{c|}{NAR} &
LLM Data Scale & NAR
\\
\hline \hline
25 & \multicolumn{1}{l|}{8.4737} & 35(7*5) & \multicolumn{1}{l|}{8.4737} & 48735(full) & 8.4737 \\
20 & \multicolumn{1}{l|}{7.3442} & 21(7*3) & \multicolumn{1}{l|}{8.3568} & 30000 & 8.3837 \\
15 & \multicolumn{1}{l|}{6.1080} & 7(7*1)  & \multicolumn{1}{l|}{8.4093} & 20000 &8.2833  \\
10 & \multicolumn{1}{l|}{4.6350} & /       &  \multicolumn{1}{l|}{/}       & 10000 & 8.3639 \\
\bottomrule
\end{tabular}
\end{table}

To investigate the impact of hyperparameters involved in our experiments, 
we conducted an in-depth analysis on candidate pool size, time segmentation, and LLM data scale. 
The results are presented in Table~\ref{tab:revision}, from which we draw three key conclusions:  (1) As the candidate pool size increases, the algorithm’s performance improves, benefiting from the greater exploration opportunities provided by a larger pool.  (2) Different levels of time segmentation lead to varying results; an appropriate segmentation that aligns with the data distribution is most beneficial. As illustrated in the word cloud analysis in the Introduction, the data exhibit distinct preference shifts across five daily periods—morning, noon, afternoon, evening, and night—thus, dividing each day into five segments proves advantageous. (3) As the scale of synthetic data used in LLM Start increases, the algorithm’s performance shows a consistent upward trend.

\section{Conclusion}

This paper introduces HyperBandit+, which enhances the initial-stage performance of online learning algorithms in periodic non-stationary streaming recommendation scenarios. Specifically, HyperBandit+ employs a hypernetwork to dynamically adjust user preference parameters for estimating time-varying rewards, thereby improving bandit policy effectiveness in online recommendation.  To address the challenge of inefficient exploration during the initial stage of online learning, we incorporate large language models (LLMs) to enrich the semantic representations of both actions and users, yielding more informative context features. Additionally, the LLM Start mechanism generates synthetic offline data, enabling a more effective warm start for bandit strategies before deployment.  To ensure efficient model training, we adopt low-rank factorization to simplify hypernetwork outputs.
A rigorous theoretical analysis proves that both the LLM Start strategy and the Hypernetwork strategy enable the bandit policy to achieve a sublinear regret bound. Furthermore, experimental results validate the effectiveness and efficiency of HyperBandit+ in streaming recommendation. Overall, HyperBandit+ presents a promising direction for integrating LLMs with controllable online learning capabilities.



\begin{acks}


This work was partially supported by the National Natural Science Foundation of China (No. 62376275, 62472426, 62502091). Work partially done at Beijing Key Laboratory of Research on Large Models and Intelligent Governance, and Engineering Research Center of Next-Generation Intelligent Search and Recommendation, MOE. Supported by the Fundamental Research Funds for the Central Universities in UIBE (Grant No. 24QN06, 24PYTS22). Supported by fund for building world-class universities (disciplines) of Renmin University of China.
\end{acks}


\bibliographystyle{ACM-Reference-Format}

\bibliography{sample-base}

\end{document}